\documentclass[sn-aps]{sn-jnl}

\usepackage{graphicx}%
\usepackage{multirow}%
\usepackage{amsmath,amssymb,amsfonts}%
\usepackage{amsthm}%
\usepackage{mathrsfs}%
\usepackage[title]{appendix}%
\usepackage{xcolor}%
\usepackage{textcomp}%
\usepackage{manyfoot}%
\usepackage{booktabs}%
\usepackage{algorithm}%
\usepackage{algorithmicx}%
\usepackage{algpseudocode}%
\usepackage{listings}%
\usepackage{subcaption}%
\usepackage{wrapfig}%
\usepackage{adjustbox}%
\usepackage{tabularx}%
\usepackage{thmtools}%
\usepackage{thm-restate}%

\theoremstyle{thmstyleone}%
\newtheorem{theorem}{Theorem}
\newtheorem{lemma}[theorem]{Lemma}%
\newtheorem{corollary}{Corollary}[theorem]%

\theoremstyle{thmstyletwo}%
\newtheorem{remark}{Remark}%

\theoremstyle{thmstylethree}%

\newcommand\independent{\protect\mathpalette{\protect\independenT}{\perp}}
\def\independenT#1#2{\mathrel{\rlap{$#1#2$}\mkern2mu{#1#2}}}

\makeatletter
\newcommand{\multiline}[1]{%
  \begin{tabularx}{\dimexpr\linewidth-\ALG@thistlm}[t]{@{}X@{}}
    #1
  \end{tabularx}
}

\begin{document}

\title[Article Title]{DAGAF: A Directed Acyclic Generative Adversarial Framework for Joint Structure Learning and Tabular Data Synthesis}


\author*[1]{\fnm{Hristo} \sur{Petkov}}\email{hristo.petkov@strath.ac.uk}

\author[1]{\fnm{Calum} \sur{MacLellan}}\email{calum.maclellan@strath.ac.uk}
\equalcont{These authors contributed equally to this work.}

\author[1]{\fnm{Feng} \sur{Dong}}\email{feng.dong@strath.ac.uk}
\equalcont{These authors contributed equally to this work.}

\affil[1]{\orgdiv{Department of Computer and Information Sciences}, \orgname{University of Strathclyde}, \orgaddress{\street{16 Richmond Street}, \city{Glasgow}, \postcode{G1 1XQ}, \state{Lanarkshire}, \country{United Kingdom}}}




\abstract{Understanding the causal relationships between data variables can provide crucial insights into the construction of tabular datasets. Most existing causality learning methods typically focus on applying a single identifiable causal model, such as the Additive Noise Model (ANM) or the Linear non-Gaussian Acyclic Model (LiNGAM), to discover the dependencies exhibited in observational data. We improve on this approach by introducing a novel dual-step framework capable of performing both causal structure learning and tabular data synthesis under multiple causal model assumptions. Our approach uses Directed Acyclic Graphs (DAG) to represent causal relationships among data variables. By applying various functional causal models including ANM, LiNGAM and the Post-Nonlinear model (PNL), we implicitly learn the contents of DAG to simulate the generative process of observational data, effectively replicating the real data distribution. 
This is supported by a theoretical analysis to explain the multiple loss terms comprising the objective function of the framework.
Experimental results demonstrate that DAGAF outperforms many existing methods in structure learning, achieving significantly lower Structural Hamming Distance (SHD) scores across both real-world and benchmark datasets (Sachs: 47\%, Child: 11\%, Hailfinder: 5\%, Pathfinder: 7\% improvement compared to state-of-the-art), while being able to produce diverse, high-quality samples.}




\maketitle

\section{Introduction}\label{sec1}

Understanding causal relationships between variables in a dataset is a crucial aspect of data analysis, as it can lead to numerous scientific discoveries. Although randomized controlled trials, which involve manipulating data through interventions, are still considered the gold standard for learning causal structures, such experiments are often impractical or even impossible due to many ethical, technical, or resource constraints. Addressing this challenge has led to a growing demand for causal studies to identify causal relationships from passive observational data.

In the last few decades, numerous approaches have emerged for performing observational causal discovery across various scientific fields, including bioinformatics \cite{Choi2020SupplementaryMO, Foraita2020CausalDO, Shen2020ChallengesAO}, economics \cite{Moneta2013CausalIB}, biology \cite{OpgenRhein2007FromCT, Londei2006ANM}, climate science \cite{EbertUphoff2012CausalDF, Runge2019InferringCF}, and social sciences \cite{Morgan2007CounterfactualsAC}. Most causal studies employ conditional independence-based algorithms, such as PC \cite{Spirtes2001CausationPA}, FCI \cite{Spirtes2000ConstructingBN}, and RFCI \cite{Colombo2011LearningHD}; discrete score-based methods like GES \cite{Chickering2003OptimalSI}, GES-mod \cite{AlonsoBarba2011ScalingUT}, and GIES \cite{Hauser2011CharacterizationAG}; or continuous optimization techniques, including NOTEARS \cite{Zheng2018DAGsWN}, DAG-GNN \cite{Yu2019DAGGNNDS}, GraN-DAG \cite{Lachapelle2020GradientBasedND}, and DAG-WGAN \cite{Petkov2022DAGWGANCS}. All these methodologies for causal structure learning have been rigorously tested and demonstrated substantial empirical evidence of their ability to produce accurate graphical representations of dependencies within datasets. However, strong performance does not necessarily resolve the issue of non-uniqueness in causal models, where multiple causal graphs can be used to define the same distribution.

To resolve the issue of non-uniqueness in causal models (e.g. Markov equivalent), where a single observed dataset may have multiple underlying structures, researchers often introduce additional assumptions \cite{peters2012identifiability}. They employ Functional Causal Models (FCM) parameterized with various structural equations to ensure that a unique causal graph is identified from a given distribution. Currently, there exist a significant amount of works that apply various identifiable (in most cases) models to learn causal structures from observational data. Noteworthy examples include the extensively researched linear non-Gaussian acyclic model (LiNGAM) \cite{Shimizu2006ALN}, the additive noise model (ANM) \cite{Hoyer2008NonlinearCD}, which provides limited support for non-linearity by assuming the relationships between variables are additive and the post-nonlinear model (PNL) \cite{Zhang2009OnTI} designed for studying complex non-linear relationships.

Among the aforementioned FCMs, the post-nonlinear (PNL) model is notable for being realistic and more accurately representing the sensor or measurement distortions commonly observed in real-world data \cite{zhang2010distinguishing}. It is also considered a superset that encompasses both ANM and LiNGAM. The PNL model consists of two functions: 1) an initial function that transforms data variables, with independent noise subsequently added to all transformations; and 2) an invertible function that applies an additional post-nonlinear transformation to each variable. Although the PNL model is one of the most general FCMs for modeling causal mechanisms in real data distributions, it is less studied than other identifiable models due to challenges associated with its post-nonlinearity and invertibility constraints.

Several approaches have been developed to investigate causal structure learning under the assumption of post-nonlinear (PNL) models, with most focusing on accurately approximating the invertibility function. For example, AbPNL \cite{uemura2022multivariate} uses an autoencoder architecture to learn a function and its inverse by minimizing a combination of independence and reconstruction loss terms. This model is applied to both bivariate and multivariate causal discovery within the context of PNL. Another approach, DeepPNL \cite{chung2019post}, parameterizes both functions of the PNL model using deep neural networks. Similarly, CAF-PoNo \cite{hoang2024enabling} employs normalizing flows to model the invertibility constraint associated with PNL. Rank-PNL, proposed by \cite{keropyan2023rank}, adapts rank-based methods to estimate the invertible function of the causal model. The latest work in this area, MC-PNL \cite{zhangpost}, aims to efficiently perform structure learning for PNL estimation by modeling nonlinear causal relationships using a novel objective function and block coordinate descent optimization. Despite recent advances in PNL estimation, causal structure learning under this functional causal model assumption remains relatively unexplored compared to other models such as ANM. 

Most existing causality learning methods typically focus on applying a single identifiable causal model to discover the dependencies exhibited in observational data. This presents a significant disadvantage as such approaches have no way to determine whether the model they assumed can learn an accurate representation of the underlying structure in a dataset. This is a critical problem to address, as misidentification of causal relationships in a dataset can result in incorrect data analysis, leading to bias in classification or inaccurate predictions. Moreover, causal discovery is also closely related to tabular data synthesis, where externally learned causal mechanisms are applied in Deep Generative Models (DGM) (e.g. DECAF \cite{Breugel2021DECAFGF}, Causal-TGAN \cite{Wen2021CausalTGANGT} and TabFairGAN \cite{Rajabi2021TabFairGANFT}) to synthesize new data samples. This method has its limitations because the accuracy of the causal knowledge must be evaluated prior to its application, which requires the availability of the true underlying structure of the dataset. This assumption proves to be impractical for real-world data, as such datasets are usually complex and extensive, with their causal structures often remaining unknown.

Recent advancements in generative modeling, including Digital Twins and transformer-based multi-attention networks, provide alternative approaches for modeling complex data relationships. Digital Twin models aim to create virtual representations of real-world systems, making them highly relevant for synthetic data generation. Similarly, attention-based architectures, such as multi-attention networks, dynamically weigh dependencies between variables. As generative models continue to gain popularity, there is significant potential to integrate them with causal discovery under a unified framework, enabling more accurate and interpretable data generation that remains faithful to underlying causal structures.


In this paper, we aim to address some of the challenges outlined above by proposing a novel framework called DAGAF, which is capable of modeling causality resembling the underlying causal mechanisms of the input data (i.e learnable causal structure approximations) and employing them to synthesize diverse, high-fidelity data samples. DAGAF learns multivariate causal structures by applying various functional causal models and determines through experimentation which one best describes the causality in a tabular dataset. Specifically, the framework supports the PNL model along with its subsets, which include LiNGAM and ANM. Unlike other methods that assume data generation is limited to a single causal model, DAGAF satisfies multiple semi-parametric assumptions. Additionally, supporting such a broad spectrum of identifiable models enables us to extensively compare our approach against the state-of-the-art in the field. We complete our study by investigating the quality of the discovered causality from a tabular data generation standpoint. We hypothesize that a precise approximation of the original causal mechanisms in a given probability distribution can be leveraged to produce realistic data samples. To prove our hypothesis, DAGAF incorporates an adversarial tabular data synthesis step, based on transfer learning, into our causal discovery framework.  


The contributions made throughout this work are outlined as follows: 

\begin{itemize}
    \item We unify causal structure learning and tabular data synthesis under a single framework capable of approximating the generative process of observational data and producing realistic samples. This approach allows us to generate quality synthetic data from the input, while preserving its causality (Section \ref{ref:sec3}). \\
    
    \item The proposed framework seamlessly integrates ANM, LiNGAM, and PNL models by leveraging a multi-objective loss function that combines adversarial loss, reconstruction loss, KL divergence, and MMD. This flexible formulation enables robust causal structure learning under diverse data-generating assumptions. Additionally, we provide a theoretical analysis to elucidate the contributions of these loss terms and how they complement each other in guiding convergence toward the true causal structure. We also analyze causal identifiability, providing conditions under which causal relationships can be uniquely determined, and examine how real-world data characteristics—such as noise, missing values, and distribution shifts—can impact identifiability (Section \ref{sec:3.1} and Section \ref{sec:3.1.6}). 
    \\
    
    \item We employ transfer learning in the context of causally-aware tabular data synthesis. DAGAF uses a two-step iterative approach that combines causal knowledge acquisition with high-quality data generation. The causal relationships identified in the first step are transferred and leveraged in the second step to facilitate causal-based tabular data generation. This enables more faithful synthetic data generation, preserving the underlying causal mechanisms (Section \ref{sec:3.2}). 
    \\

    
    \item We validate the effectiveness of DAGAF on synthetic, benchmark, and real-world datasets. Our results show significant improvement in DAG learning in comparison with other methods (Sachs: 47\%, Child: 11\%, Hailfinder: 5\%, Pathfinder: 7\% improvement compared to state-of-the-art). They also demonstrate that the learned causal mechanism approximations can be used to generate high-quality, realistic data (Section \ref{sec:5}). 
\end{itemize}

\section{Prerequisites}\label{sec:2}

This section explores the mathematical aspects of causality, relevant to the field of machine learning. In particular, we provide a brief overview of Functional Causal Models (FCM) \cite{pearl2009causality} and the assumptions employed in our causal structure learning algorithm. 

Let $\chi$ denote a tabular dataset such that $X = \{X_1, \dots, X_d\}$ is a set of $d$ random data variables, and $\chi \subseteq \mathbb{R}^{n \times d}$ represents a dataset consisting of $n$ samples $\mathbf{X} = \{\mathbf{X}_1, \dots ,\mathbf{X}_n\}$ drawn from the joint distribution $P(\mathbf{X})$. Individual data points and their attributes are denoted as $\mathbf{X}_i$ and $X_j$, respectively. Additionally, let $\mathcal{G}_\mathcal{A} \in \mathbb{D}$ be a ground truth Directed Acyclic Graph (DAG) representing the relationships between all the attributes $\{X_1, \dots, X_d\}$. Then, $P(\mathbf{X})$ can be expressed using a functional causal model (FCM), which describes the relationships within $\{X_1, \dots, X_d\}$. In this context, FCMs facilitate causal discovery from tabular datasets by encoding variables as nodes, and edges between them represent the underlying causal mechanisms responsible for data generation.

According to theory, an FCM is formulated as a triplet $\mathcal{M}_{\mathcal{G}_\mathcal{A}} \langle X, \mathcal{F}, \mathcal{Z} \rangle$, where $X = \{X_1, \dots, X_d\}$ is a set of endogenous variables, $\mathcal{F} = \{f_1, \dots, f_d\}$ is a set of structural equations, and $\mathcal{Z} = \{\mathcal{Z}_1, \dots, \mathcal{Z}_d\}$ is a set of exogenous (noise) variables. Under the local Markov condition and the causal sufficiency assumption, the joint distribution of $\mathbf{X}$ can be factorized as $P(\mathbf{X})= \prod_{j=1}^{d}P(X_j\mid{Pa}_j)$, where $X_j$ is a child of its parent variables ${Pa}_j$ in the graph $\mathcal{G}_\mathcal{A}$. Each $X_j$ can be modeled in its non-parametric form as:

\begin{equation}
\label{eq1}
    X_j := f_j(\text{Pa}_j, \mathcal{Z}_j).
\end{equation}
This representation of $P(\mathbf{X})$ allows us to sequentially model the causal mechanisms underlying $\chi$, defining its generative process.

Furthermore, we assume faithfulness, which enables the discovery of causal structures from continuous observational data using various nonlinear and semi-parametric models. Our framework is applied to several types of models, including: Linear non-Gaussian Acyclic Models (LiNGAM), Additive Noise Models (ANM), and Post-Nonlinear Models (PNL). Each of these models has been proven to be causally identifiable under specific assumptions: 

\begin{itemize}
    \item \textbf{LiNGAM}: The causal identifiability of LiNGAM is guaranteed under the assumption of non-Gaussianity in the noise terms. Specifically, if the noise variables are non-Gaussian and independent from $X$, it has been shown that the underlying causal structure can be uniquely identified \cite{Shimizu2006ALN}.
    
    \item \textbf{ANM}: Additive Noise Models (ANM) assume that the Gaussian noise term $\mathcal{Z}_j$ is independent of the parent variables $Pa_j$. This assumption enables the identifiability of the causal direction between variables. Additionally, the function $f_j(\cdot)$ must be non-linear and three times differentiable, to ensure that the application of this model results in a unique determination of the causal direction between variables \cite{Hoyer2008NonlinearCD}.
    
    \item \textbf{PNL}: Post-Nonlinear Models (PNL) extend the ANM framework by introducing an additional non-linear transformation $g_j(\cdot)$ after the function $f_j(\cdot)$. The key assumptions for identifiability in PNL include the independence of the Gaussian noise terms and the non-linear and invertible nature of the function $g_j(\cdot)$. Under these conditions, the causal structure can be identified, even in the presence of complex non-linear interactions \cite{Zhang2009OnTI}.
\end{itemize}

\section{DAGAF: a general framework for simultaneous causal discovery and tabular data synthesis}
\label{ref:sec3}

DAGAF learns DAG structures from input data to simulate the generative process of their probability distribution. We model $G_A$ to represent the causal relationships within a dataset $\chi$. The model is capable of facilitating realistic sample synthesis with minimal loss of fidelity and diversity. We formalize our goal as follows. \\

\noindent \textbf{Goal}: Given $n$ i.i.d. observations $\mathbf{X} \sim P(\mathbf{X}) \in \chi$, we propose a general framework to learn $G_A \approx \mathcal{G}_\mathcal{A} \in \mathbb{D}$ together with a set of structural equations $\mathcal{F} = \{f_1,...f_d\}$, such that $\tilde{X}_j := f_j({Pa}_j, \mathcal{Z}_j)$ yields $\tilde{\mathbf{X}} \sim P_{G_A}(\tilde{\mathbf{X}}) \in \tilde{\chi}$ matching the input. \\

The DAGAF framework focuses on learning an approximation of the causal mechanisms $\{f_j({Pa}_j, \mathcal{Z}_j)\}$ involved in the generation of observations $\mathbf{X}$. The (semi)parametric assumptions outlined in Section \ref{sec:2} define each node $X_j \in G_A$ as a function $f_j : \mathbb{R}^d \rightarrow \mathbb{R}$. Under such circumstances, the general nonparametric form $\mathbb{E}[X_j|X_{pa(j)}] := \mathbb{E}_\mathcal{Z}(f_j(X, \mathcal{Z}))$ can be reduced to one of the following: 1) \textbf{Linear non-Gaussian Acyclic Models (LiNGAM)}: $\tilde{X} := f(X) + \mathcal{Z}$, where $f(X)$ is a linear function of $X$ and $\mathcal{Z}$ is a non-Gaussian noise term independent of $X$; 2) \textbf{Additive Noise Models (ANM)}: $\tilde{X}_j := f_j(Pa_j) + \mathcal{Z}_j$, 
where $f_j$ is a nonlinear function of the parent variables $Pa_j$, and $\mathcal{Z}_j \independent f_j(Pa_j), \mathcal{Z}_j \sim \mathcal{N}(\mu, \sigma^2_j)$; 3) \textbf{Post-Nonlinear Models (PNL)}: $\tilde{X}_j := g_j(f_j(Pa_j) + \mathcal{Z}_j)$, where $g_j$ is a nonlinear function and $\mathcal{Z}_j \independent f_j(Pa_j), \mathcal{Z}_j \sim \mathcal{N}(\mu, \sigma^2_j)$.

Algorithm \ref{alg:cap} provides an overview of the training process. Section \ref{sec:3.1} details Step 1, which focuses on causal structure learning. Furthermore, since the framework recovers the causal structure by learning the underlying data generative process of $\mathbf{X}$, it is naturally well-suited for data synthesis. However, it requires training a separate Deep Generative Model (DGM) involving a discriminator and a generator in an additional training phase, which is explained in detail in Section \ref{sec:3.2}. The architecture and training procedure of DAGAF are described in Section \ref{sec:3.3}. A visual representation of the model pipeline is provided in Figure \ref{fig:3.1}. 

\begin{algorithm} [!htp]
\caption{DAGAF training algorithm}
\label{alg:cap}
\begin{algorithmic}
\scriptsize
\Require Sample $n$ observational data points $\{\mathbf{X}_1, \dots ,\mathbf{X}_n\}$ from the training data and $d$ noise vectors $\{Z_1, \dots ,Z_d\}$ from normal or uniform distributions. Generate $n$ synthetic data samples $\{\tilde{\mathbf{X}}_1, \dots ,\tilde{\mathbf{X}}_n\}$, with data attributes $\tilde{X} := f(X) + \mathcal{Z}$, $\tilde{X}_j := f_j({Pa}_j) + \mathcal{Z}_j$ or $\tilde{X}_j := g_j(f_j({Pa}_j) + \mathcal{Z}_j)$ depending on whether LiNGAM, ANM or PNL is assumed. \\
\Ensure The acyclicity constraint value $h(A^{L_0}(f))$ is higher than its tolerance of error $h\_tol$ set to 1e-8. Each step during training has its own instance of DAG-Notears-MLP. Causal information is transferred from the FCM into the DGM architecture.\\
\State \textbf{Step 1}: Poly-assumptive causal structure learning
\State \qquad LiNGAM, ANM $\rightarrow$ learn $f$ by minimizing a combination of loss terms including \State \qquad adversarial loss \eqref{eq2}, Mean Squared Error \eqref{eq3}, Kullback-Lieber divergence \eqref{eq:4}, 
\State \qquad Maximum Mean Discrepancy \eqref{eq:5} and the acyclicity constraint from \cite{Zheng2019LearningSN} 

\State \qquad PNL $\rightarrow$ learn $f$ using the loss terms described in the LiNGAM, ANM case and 
\State \qquad $g^{-1}$ by solving \eqref{eq9}
\State \qquad This step recovers a graph representation $G_A$ of the causal mechanisms in $\mathbf{X}$.\\

\State \textbf{Step 2}: Generative process simulation under multiple causal model assumptions

\State \qquad LiNGAM, ANM $\rightarrow$ learn $f$ by computing \eqref{eq2}
\State \qquad PNL $\rightarrow$ learn $f$ and $g$ by finding the optimal value for \eqref{eq2}
\State \qquad This step models a generative process involving $G_A$ through adversarial 
\State \qquad training, producing new data samples. 
\end{algorithmic}
\end{algorithm}

\begin{figure*} [!ht]
  \centering
  \includegraphics[width=11.4cm, keepaspectratio]{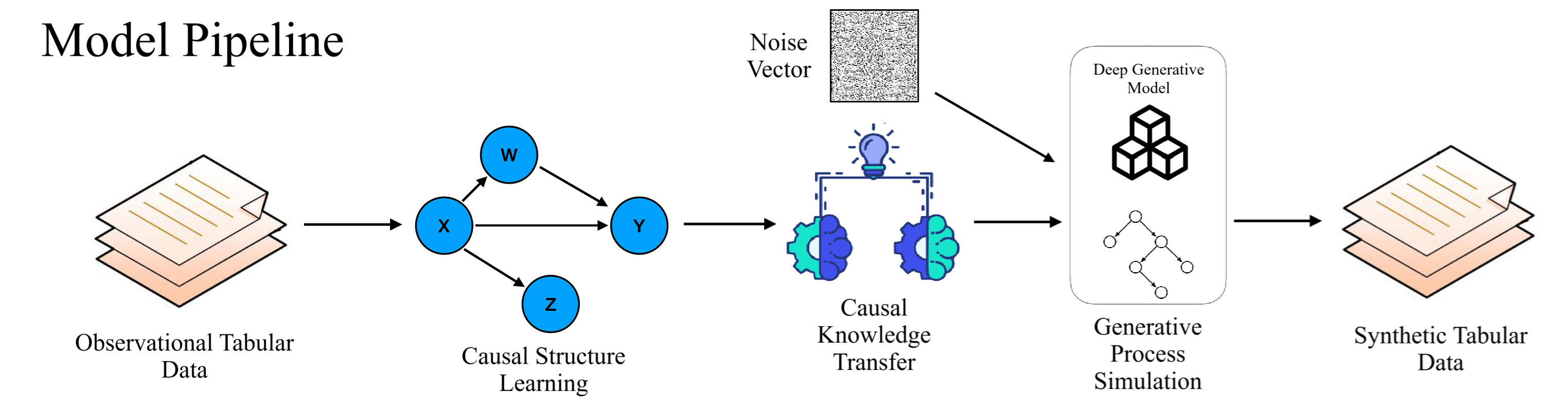}
  \caption{Pipeline of the DAGAF algorithm} 
  \label{fig:3.1}
\end{figure*}

\subsection{Loss functions for causal structure learning}
\label{sec:3.1}

In Step 1 of DAGAF training, the goal is to model DAGs using a sophisticated objective function that integrates a combination of loss terms used for causal structure learning. In its basic form, the framework covers LiNGAM and ANM by utilizing adversarial training and reconstruction loss, along with some regularization terms, to learn how to generate $\tilde{\mathbf{X}}$ from $\mathbf{X}$. One benefit of our framework is its flexibility, allowing the basic approach to be easily adapted to support causal structure learning using PNL. The advanced form further extends the functionality of the framework to cover PNL by adding an additional reconstruction loss term to model the non-linear function $g_j$.

\subsubsection{Adversarial loss with gradient penalty}
\label{sec:3.1.1}

DAGAF simulates $\mathbf{X}$ by learning how to generate $\tilde{\mathbf{X}}$ using causal mechanism approximations of $\{f_j({Pa}_j, \mathcal{Z}_j)\} \in P(\mathbf{X})$. To achieve this, we do not directly model $\tilde{\mathbf{X}}$ but instead focus on recovering the causal mechanisms $\mathcal{F} = \{f_1, \dots, f_d\}$, where each $f_j$ is defined as $f_j({Pa}_j; W^1_j, \dots, W^L_j) + \mathcal{Z}_j$.
Learning the causal mechanisms involves determining the immediate parents of each variable, which are encoded in the causal structure of $\mathbf{X}$. We minimize  the Wasserstein distance $\mathbb{W}_p(P(\mathbf{X}), P_{G_A}(\tilde{\mathbf{X}}))$  through adversarial training, which implicitly refines the causal structure $G_A$, facilitating the identification of the causal mechanisms. The Wasserstein distance with gradient penalty loss term is defined as follows:

\begin{equation}
\label{eq2}
\begin{aligned}
    \mathcal{L}_{\text{adv}}(\mathbf{X}, \tilde{\mathbf{X}}) &=  
    \sup_{\|\phi\|_L \leq 1} \mathbb{E}_{\mathbf{X} \sim P(\mathbf{X})}[\phi(\mathbf{X})] - \mathbb{E}_{\tilde{\mathbf{X}} \sim P_{G_A}(x \mid G)}[\phi(\tilde{\mathbf{X}})] 
    \\ 
    &= \mathbb{E}_{\mathbf{X} \sim P(\mathbf{X})}[D(\mathbf{X})] - \mathbb{E}_{\tilde{\mathbf{X}} \sim P_{G_A}(\tilde{\mathbf{X}})}[D(\tilde{\mathbf{X}})] \\ 
    &+ \mathbb{E}_{\hat{\mathbf{X}} \sim P(\hat{\mathbf{X}})}[(||\nabla_{\hat{\mathbf{X}}}D(\hat{\mathbf{X}})||_2 - 1)^2],
\end{aligned}
\end{equation} where $\phi(\mathbf{X})$ is a 1-Lipschitz function used to approximate the Wasserstein distance $\mathbb{W}p(P(\mathbf{X}), P_{G_A}(\tilde{\mathbf{X}}))$. The function $D(\mathbf{X})$ serves as the discriminator, which is trained adversarially to learn $\phi(\mathbf{X})$ and distinguish between real and generated samples.

In this framework, adversarial training to optimise \eqref{eq2} involves learning the set of structural equations $\mathcal{F} = \{f_1, \dots, f_d\}$, where each $f_j$ models the causal mechanism of node $X_j$. The FCM-based generator $\mathcal{M}$ learns to generate synthetic data that mimics the true distribution, while the discriminator $D(\mathbf{X})$ evaluates the divergence between real and generated samples. The objective is formulated as a min-max optimization, where $\mathcal{M}$ aims to minimize the discrepancy measured by $D(\mathbf{X})$, while $D(\mathbf{X})$ is trained to distinguish between real and generated distributions, typically using the Wasserstein distance. Theoretically, this min-max optimization problem achieves its optimal point typically characterized as a Nash equilibrium, when the generator can yield synthetic data that is indistinguishable from $\mathbf{X}$, thereby approximating the generative process of $\mathbf{X}$ (i.f.f. the causal structure in $G_A$ is correctly identified). \\

\begin{restatable}{proposition}{primepropone}
\label{prop1}
    Let the ground-truth DAG $\mathcal{G}_\mathcal{A}$ be uniquely identifiable from $P(\mathbf{X})$, then, under the causal identifiability assumption, minimizing adversarial loss ensures that the implicitly generated distribution $P_{G_A}(\tilde{\mathbf{X}})$ aligns with $P(\mathbf{X})$.

    \[
        \mathbb{W}_p(P(\mathbf{X}), P_{G_A}(\tilde{\mathbf{X}})) = 0 \to P_{G_A}(\tilde{\mathbf{X}}) = P(\mathbf{X}) \Longleftrightarrow G_A = \mathcal{G}_\mathcal{A}.
    \]
\end{restatable}

\begin{proof}
    The proof of Proposition \ref{prop1} is available in Appendix \ref{a1}.
\end{proof}

\color{black}


\subsubsection{Reconstruction loss}

We add a reconstruction loss to enhance causal structure learning. In this context, we use Mean Squared Error (MSE) as the reconstruction loss:

\begin{equation}
        \label{eq3}
        \mathcal{L}_{\text{MSE}}(\mathbf{X}, \tilde{\mathbf{X}}) =
        \mathbb{E}_{\mathbf{X}, \tilde{\mathbf{X}}}(||\mathbf{X}-\tilde{\mathbf{X}}||_2)=\frac{1}{n} \sum\limits_{i=1}^{n} \sum\limits_{j=1}^{d}||X_{ij} - {\{f_j({Pa}_j;W^1_j,...,W^L_j)+\mathcal{Z}_j\}}_i||_2
\end{equation}
By reducing \eqref{eq3} through parameter optimization, we minimize the residual distance between individual samples $||\mathbf{X} - \tilde{\mathbf{X}}||$ such that our model produces $\tilde{\mathbf{X}} \sim P_{G_A}(\tilde{\mathbf{X}})$ by implicitly learning the causal dependencies of $\mathbf{X}$ represented in $G_A$. Essentially, this reconstruction process results in a closer approximation of the causal mechanisms responsible for producing $\mathbf{X}$. \\

\begin{restatable}{proposition}{primeproptwo}
\label{prop2}
    The MSE loss ensures point-wise alignment between the data and the prediction of the model, improving the smoothness of the gradient and the stability of adversarial optimization.
    \[\inf_{G_A \in \mathbb{D}} \mathcal{L}_{\text{MSE}}(\mathbf{X}, \tilde{\mathbf{X}}) =0 \Rightarrow \forall i, \tilde{\mathbf{X}}_i = \mathbf{X}_i  \]
\end{restatable}

\begin{proof}
    The proof of Proposition \ref{prop2} is available in Appendix \ref{a2}.
\end{proof}
\color{black}
The MSE loss plays a key role in DAG learning, as evidenced by our experiments. This aligns with the approach taken by most existing works in DAG-learning, where MSE is the most commonly used loss function.

\subsubsection{Kullback–Leibler Divergence}

We introduce Kullback–Leibler divergence (KLD) \cite{Kullback1951OnIA} as a regularization term for nonlinear cases with additive Gaussian noise in ANM to prevent overfitting of $\mathbf{X}$ and inaccurate causal mechanisms in the generative process of $\tilde{\mathbf{X}}$. The KLD term is typically applied in Variational Autoencoders (VAE) as a regularization component of the Evidence Lower Bound (ELBO) loss function for latent variables. It is defined as $D_{KL}\left(\mathcal{N}(\mu, \sigma^2) \| \mathcal{N}(0, 1)\right) = \frac{1}{2} \sum_{i=1}^n \left(\sigma_i^2 + \mu_i^2 - \log(\sigma_i^2) - 1\right)$ where $\mu$ and $\sigma$ denote the mean and standard deviation of $\tilde{\mathbf{X}}$. In our setup, we apply this to regularize $\tilde{\mathbf{X}}$. Additionally, we only model the mean of $P_{G_A}(\tilde{\mathbf{X}})$ and set its variance to 1, hence reducing the regularization function to:

\begin{equation}
\label{eq:4}
     \mathcal{L}_{\text{KLD}}(\mathbf{X}, \tilde{\mathbf{X}}) = D_{KL}(P(\mathbf{X})||P_{G_A}(\tilde{\mathbf{X}})) = \frac{1}{2} \sum\limits_{i=1}^{n}(\mu^2_i).
\end{equation} 

We use the Kullback–Leibler divergence (KLD) as a regularization term for $\tilde{\mathbf{X}}$, the model-generated data, to simulate an additive noise scenario where noise is incorporated into each data point. By applying KLD to $\tilde{\mathbf{X}}$, we encourage the model to produce $\tilde{\mathbf{X}}$ that closely matches the true data distribution while accounting for the variability introduced by noise. This regularization helps the model avoid overfitting by ensuring that the generated data reflects the natural variations present in the real data, leading to more robust and realistic samples. As our model involves learning causal mechanisms, this prevents the model from learning incorrect causal structures, such as misidentifying child nodes as parent nodes.\\

\begin{restatable}{proposition}{primepropthree}
\label{prop3}
    The $\mathcal{L}_{\text{KLD}}(\mathbf{X}, \tilde{\mathbf{X}})$ regularization provides a statistical prior on the learned distribution $P_{G_A}(\tilde{\mathbf{X}})$, ensuring it adheres to a Gaussian assumption. It also acts as a stabilizing factor in optimization, particularly under the additive Gaussian noise model. It complements the adversarial and MSE losses, ensuring both alignment and smoothness of $P_{G_A}(\tilde{\mathbf{X}})$.
\end{restatable}

\begin{proof}
    The proof of Proposition \ref{prop3} is available in Appendix \ref{a3}.
\end{proof}
\color{black}

\noindent  Note, this is not applicable to the LiNGAM causal model, due to the non-Gaussianity of the noise term $\mathcal{Z}$ under that specific assumption.

\subsubsection{Maximum Mean Discrepancy}
\label{sec:3.1.4}

The reconstruction loss and its regularization term focus solely on learning the mean of $P(\mathbf{X})$, while completely disregarding its variance. This implies that the reconstruction process involved in DAGAF is highly sensitive to rare occurrences (i.e. outliers) in $P(\mathbf{X})$. To address this issue, we further reduce the residual distance between the input distribution $\mathbf{X} \sim P(\mathbf{X})$ and the generated data distribution $\tilde{\mathbf{X}} \sim P_{G_A}(\tilde{\mathbf{X}})$ by introducing the Maximum Mean Discrepancy (MMD) \cite{Tolstikhin2016MinimaxEO}. We apply the kernel trick \cite{khemakhem2021causal} to compute the solution to this formula.

\begin{equation}
\label{eq:5}
    \begin{aligned}
         \mathcal{L}_{\text{MMD}}(\mathbf{X}, \tilde{\mathbf{X}}) &=  ||\mathbb{E}_{\mathbf{X} \sim P(\mathbf{X})}[k(\mathbf{X})] - \mathbb{E}_{\tilde{\mathbf{X}} \sim P_{G_A}(\tilde{\mathbf{X}})}[k(\tilde{\mathbf{X}})]||^2_\mathcal{H} \\
        &= \frac{1}{n}\sum\limits_{i \neq j}^{n}k(\mathbf{X}_i,\mathbf{X}_j) - \frac{2}{n}\sum\limits_{i \neq j}^{n}k(\mathbf{X}_i, \tilde{\mathbf{X}}_j) + \frac{1}{n}\sum\limits_{i \neq j}^{n}k(\tilde{\mathbf{X}}_i, \tilde{\mathbf{X}}_j),
    \end{aligned}
\end{equation} where $\mathcal{H}$ denotes the reproducing kernel Hilbert space (RKHS) and $k \in \mathcal{H}$ is a kernel function.

The MMD maximizes mutual information between $P(\mathbf{X})$ and $P_{G_A}(\tilde{\mathbf{X}})$, leading to alignment in both their means and overall shapes. Specifically, by matching the shapes of the distributions, the MMD term can help bring their variances closer together. Hence, by applying \eqref{eq:5} we indirectly model the standard deviation of $P_{G_A}(\tilde{\mathbf{X}})$ to mitigate mode collapse in $\mathbf{\tilde{X}}$ and discover the causal mechanisms responsible for producing its outliers. \\

\begin{restatable}{proposition}{primepropfour}
\label{prop4}
    Minimizing the Maximum Mean Discrepancy (MMD) loss $\mathcal{L}_{\text{MMD}}(\mathbf{X}, \tilde{\mathbf{X}})$ aligns higher-order statistics of $P(\mathbf{X})$ and $P_{G_A}(\tilde{\mathbf{X}})$, complementing adversarial loss to achieve overall distributional alignment.
\end{restatable}

\begin{proof}
    The proof of Proposition \ref{prop4} is available in Appendix \ref{a4}.
\end{proof}

Our ablation study in Appendix \ref{ab} indicates that the MMD term incorporated from DAG-GAN \cite{9414770} makes contributions to causal discovery.

\subsubsection{Post-Nonlinear FCM}

So far, we have discussed the loss terms for the LiNGAM and ANM cases, where $\tilde{\mathbf{X}}$ generated using causal mechanism approximations $\tilde{X} := f(X) + \mathcal{Z}$ or $\tilde{X}_j = f_j({Pa}_j) + \mathcal{Z}_j$ is treated as the final output of the model to mimic the training data $\mathbf{X}$ via minimizing $||P(\mathbf{X}) - P_{G_A}(\tilde{\mathbf{X}})||$. One of the key advantages of DAGAF is its flexibility, allowing this to be extended to handle Post-Nonlinear Models (PNL). 

PNL is crucial for causal discovery as it provides a more realistic approach to modeling causality by capturing non-linear effects in observational data. Furthermore, PNL is considered a general superset that encompasses other identifiable models, such as ANM \cite{Peters2013CausalDW} and LiNGAM \cite{Shimizu2006ALN}.

\begin{equation}
\label{eq4}
    \tilde{X}_j := g_j(f_j({Pa}_j) + \mathcal{Z}_j), \forall j, \mathcal{Z}_j \independent f_j({Pa}_j)
\end{equation} 

Without loss of generality, we rearrange \eqref{eq4} into

\begin{equation}
    \mathcal{Z}_j = g_j^{-1}(\tilde{X}_j) - f_j({Pa}_j),
\end{equation} where $g^{-1}$ is the inverse of $g$. Under this setting (from the rearranged equation), the problem has been broken into two parts, which are to learn $f(\cdot)$ and $g^{-1}(\cdot)$ respectively.

Learning $f(\cdot)$ follows the same process as in the ANM and LiNGAM cases, as described so far in Section \ref{sec:3.1.1} to Section \ref{sec:3.1.4}. However, learning $g^{-1}(\cdot)$ is an additional step specific to the PNL case. In practice, both functions $g^{-1}(\cdot)$ and $f(\cdot)$ are modeled using two different neural networks, where $f(\cdot)$ is the same as before and $g^{-1}(\cdot)$ is the inverse of a general MLP. There is an additional Mean Squared Error (MSE) term involved in the training procedure, which we define as:

\begin{equation}
\label{eq9}
     \mathcal{L}_{\text{PNL}}(\hat{\mathbf{X}}, \tilde{\mathbf{X}}) =  MSE(\hat{\mathbf{X}}, \tilde{\mathbf{X}}) = \frac{1}{n} \sum\limits_{i=1}^{n} \sum\limits_{j=1}^{d}||{g_j^{-1}(X_j)}_i - {f_j({Pa}_j)}_i||_2,
\end{equation} where $\hat{\mathbf{X}}$ is the output of $g^{-1}$.


It is worth noting that the reason why the loss terms in Sections \ref{sec:3.1.1}-\ref{sec:3.1.4} (where $f(.)$ is treated as the final output of the model) can be used by the PNL case is based on the idea of skip connections, as those used in ResNet. Although the output from $f(\cdot)$ in the PNL case is not the final output, we can still use it directly in these loss terms by essentially skipping the final function $g(\cdot)$, allowing the model to apply the same loss terms as in the ANM and LiNGAM cases. For more information on this concept, see \cite{He2015DeepRL}.

\subsubsection{Causal structure acyclicity}

Minimizing the reconstruction and adversarial loss terms does not guarantee that $G_A$ will be acyclic. To prevent cycles from occurring in the learned causal structures, we employ the implicit acyclicity constraint from \cite{Zheng2019LearningSN} $h(A^{L_0}(f))=0$, where $A^{L_0} \in \mathbb{R}^{d \times d}$ is the weighted adjacency matrix described implicitly by the model weights. More details can be found in \cite{Zheng2019LearningSN}.

\subsection{Simulating data generative processes}
\label{sec:3.2}

In the second stage of Algorithm \ref{alg:cap}, we focus on synthesizing realistic tabular data samples using the causal graph $G_A$ produced from Step 1. Our data generation process assumes a different instance of the FCM $\mathcal{M}$ used in the causal discovery step, which we refer to as generator $G$ here. Causal knowledge is transferred between FCM instances by loading $W^{L_0}$ from $\mathcal{M}$ into $L_0 \in G$. To enable tabular data synthesis, we incorporate an additional noise vector $Z = \{Z_1,...Z_d\} \sim \mathcal{N}(\mu, \sigma^2)$ into the architecture of the generator.

The models used in this step are trained adversarially to ensure that $P_{G_A}(\tilde{\mathbf{X}})$ closely approximates $P(\mathbf{X})$. Specifically, the network $G$ creates samples while competing against a discriminator $D: \mathbb{R}^d \rightarrow \mathbb{R}$, whose aim is to distinguish between synthetic samples and observational samples. We apply Wasserstein-1 with gradient penalty to train our DGM, resulting in realistic samples indistinguishable from $\mathbf{X}$. The loss function is the same as Equation \eqref{eq2}. More specifically, we consider each connected layer $\alpha(L_j) \in \{\alpha(L_1),...\alpha(L_d)\}$ as an individual generator $G_j(Z_j) \in \{G_1(Z_1),...G_d(Z_d)\}$. This approach enables us to model each causal mechanism $f_j \in \{f_1,...f_d\}$ such that $\tilde{X}_j$ is generated as either $\tilde{X} := G(X) + \mathcal{Z}$; $\tilde{X}_j := G_j({Pa}_j) + Z_j$ or $\tilde{X}_j := g_j(G_j({Pa}_j)+Z_j)$, depending on whether we assume LiNGAM, ANM or PNL. In other words, we generate a synthetic tabular dataset $\tilde{\mathbf{X}} \in \tilde{\chi} \subseteq \mathbb{R}^{n \times d} = \mathcal{F}(Z) = \{f_j({Pa}_j, Z_j)\}$. During training, we only update the parameters $W = \{W^1,...,W^L\}$ of the locally connected hidden layers, since modifying the weights of $L_0$ would affect the structural equations $\mathcal{F}$ used to produce $\tilde{\mathbf{X}}$.

Our experiments in Section \ref{ref:sec4.4} indicate that our DMG can produce high-quality data under both the ANM and PNL structural assumptions. 

\subsection{Model architecture and training specifications}
\label{sec:3.3}

Figure~\ref{fig:3.2} presents the overall architecture of the DAGAF framework. Figure~\ref{fig:3.2}a illustrates the ANM and LiNGAM setting, where input data $\mathbf{X}$ is processed by function $f$ to produce $\hat{\mathbf{X}}$. The optimization is guided by multiple loss terms: $L_{\text{adv}}(\mathbf{X}, \tilde{\mathbf{X}})$, $L_{\text{MSE}}(\mathbf{X}, \tilde{\mathbf{X}})$, $L_{\text{KLD}}(\mathbf{X}, \tilde{\mathbf{X}})$, and $L_{\text{MMD}}(\mathbf{X}, \tilde{\mathbf{X}})$, with $L_{\text{KLD}}(\mathbf{X}, \tilde{\mathbf{X}})$ specifically excluded in the LiNGAM case. 
Figure~\ref{fig:3.2}b extends Figure~\ref{fig:3.2}a by incorporating the PNL model. The right-hand branch follows the same structure as Figure~\ref{fig:3.2}a, while the additional left-hand branch applies $g^{-1}$ to invert $\mathbf{X}$. This inversion contributes to computing $L_{\text{PNL}}(\mathbf{X}, \tilde{\mathbf{X}})$, which is then integrated with the other loss terms from the right-hand branch, forming a unified optimization framework. 
Figure~\ref{fig:3.2}c depicts the data generation process used to synthesize artificial data, demonstrating how the framework facilitates structured data synthesis.

We incorporate the Multi-Layer Perceptron (MLP) from \cite{Zheng2019LearningSN} as an FCM $\mathcal{M}$ to model $f$ in the causal structure learning step. Its architecture consists of two components: 1) an initial linear layer $L_0$, which constitutes an implicit definition of $G_A$, enabling the modelling of causal structures and 2) a set of locally connected hidden layers $L = \{\alpha(L_1), ..., \alpha(L_d)\}$, with $\alpha$ being a nonlinear transformation applied to each layer, designed to approximate and learn $\mathcal{F} = \{f_1,...,f_d\} \in G_A$. Meanwhile, $g$ is a general MLP with five linear layers [$d$ - 10$ d$ - 10$d$ - 10$d$ - $d$] (1 input, 3 hidden and 1 output) and nonlinearity applied using the ReLU activation function (only used in the PNL case). More specifically, each feature in $\mathbf{X}$ is modeled with a neural network of $L$ hidden layers $f_j({Pa}_j,\mathcal{Z}_j;W^1_j,...,W^L_j), j \in [1, d]$, where $W^l_j$ denotes the parameters of the $l^{th}$ layer. Let \( W^{(0)}_j \in \mathbb{R}^{h \times d} \) be the weight matrix within $L_0$ connecting to the local neural network modeling \( X_j \), where \( h \) is the latent size and \( d \) is the number of input variables. For each pair of variables \( X_j \) and \( X_k \), the Ridge regression norm of the weights connecting \( X_k \) to all latent units in the network for \( X_j \) is computed as: 


\begin{equation}
    A_{jk} = \left\| W^{(1)}_{j,k,:} \right\|_2 =\sqrt{\sum_{m=1}^{h} \left( W^{(1)}_{j,k,m} \right)^2},
\end{equation} where \( W^{(1)}_{j,k,m} \) represents the weight connecting the \( k \)-th input variable \( X_k \) to the \( m \)-th latent unit in the first layer of the network for \( X_j \).

\begin{figure*}[!ht]
  \centering
  \includegraphics[width=12.5cm, keepaspectratio]{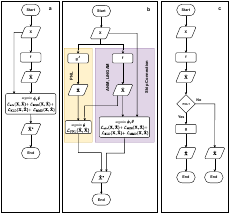}
  \caption{A Visual Representation of DAGAF. (a) The optimization structure under ANM and LiNGAM, where input data is processed to reconstruct \(\tilde{\mathbf{X}}\) using multiple loss terms, excluding \(L_{\text{KLD}}\) in the LiNGAM case. (b) The extended framework integrating ANM, LiNGAM, and PNL, where an additional inversion function \(g^{-1}\) is introduced to compute \(L_{\text{PNL}}\), unifying the optimization process. (c) The synthetic data generation process, illustrating how the framework enables structured data synthesis while preserving underlying causal relationships.}
  \label{fig:3.2}
\end{figure*}


Throughout the training process, a learning rate of $3 \times 10^{-3}$ is employed, with a batch size set at 1000. Ridge regression regularization is applied in both steps by setting the weight decay of both discriminators to $1 \times 10^{-6}$. The models within our framework undergo iterative optimization, with their parameters updated through gradient descent.

The adversarial loss is applied to the reconstructed distribution $P_{G_A}(\tilde{\mathbf{X}})$, hence, in the causal structure learning step, a noise vector is not involved during training. Once the parameters in $A^{L_0}$ have been updated, we convert $A^{L_0}$ to $G_A$ using the post-processing step $G_A = \sqrt{A^{L_0}(f)}, w^2_{jk} \in A^{L_0}(f)$ followed by thresholding with value 0.3, considered best by existing works such as DAG-GNN \cite{Yu2019DAGGNNDS}, GAE \cite{Ng2019AGA}
and many others. These final two steps are required to recover the weights $w_{jk} \in G_A$ from $A^{L_0}(f)$ and to reduce the number of false discoveries in $G_A$.

To learn $g^{-1}$ for the PNL case, we need to invert the architecture and training procedure of $g$ such that $\tilde{\mathbf{X}}$ is used as input to produce the original $\mathbf{X}$. We opt to focus on the training algorithm only as due to the generality of $g$ inverting its architecture will not result in any changes to its configuration. \\

\begin{remark}
The output data $\tilde{\mathbf{X}}$ from Step 1 is solely used to compute the loss terms during training and then it is discarded. This happens because the reconstruction loss used to learn the causal structure of $\mathbf{X}$ significantly reduces the range of the generated samples, resulting in $\tilde{\mathbf{X}}$ with high fidelity but low diversity.\\
\end{remark} 

We treat the training as a constraint continuous optimization problem because of the requirement to adjust the parameters of the acyclicity constraint together with the weights of the model. Hence, we use the modified version of the augmented Lagrangian \cite{Bertsekas:jair-1999} employed in DAG-Notears-MLP to solve it. 

\subsection{Computational Complexity Analysis}

The DAGAF framework comprises three distinct models: the FCM/Generator ($\mathcal{M}/G$), the Discriminator $D$ (in the ANM and LiNGAM settings), and an additional MLP $g$ for the PNL case. These models are trained using an algorithm that integrates three interconnected components: Causal Structure Learning, Tabular Data Synthesis, and Augmented Lagrangian-based Continuous Optimization. This complex architecture and training methodology make DAGAF significantly more intricate compared to other state-of-the-art methods, such as DAG-GNN \cite{Yu2019DAGGNNDS}, GraN-DAG \cite{Lachapelle2020GradientBasedND}, DECAF \cite{Breugel2021DECAFGF}, and Causal-TGAN \cite{Wen2021CausalTGANGT}, which focus solely on causal discovery or tabular data synthesis and involve fewer models. This complexity motivated us to assess the efficiency and practicality of our approach.

We examine the resource requirements of DAGAF for performing causal structure learning and tabular data synthesis simultaneously. To achieve this, we provide pseudo-code for Algorithm \ref{alg:cap} and analyze its time complexity. This alternative representation of the training process for our framework is presented in Appendix \ref{ae}. The space complexity of DAGAF is $\mathcal{O}(d)$, where $d$ represents the number of variables in $\mathbf{X}$, aligning with the complexity of Notears and its extensions. 

To perform a thorough time complexity analysis of our framework, we evaluate the efficiency of each stage in the pseudo-code from Appendix \ref{ae} separately. This analysis also incorporates the augmented Lagrangian and causal knowledge transfer components. The total computational complexity is determined by summing the individual complexities of each component in the pseudo-code for Algorithm \ref{alg:cap} and identifying the most resource-intensive stage. We start with the initial phase of the framework, which involves declaring variables, hyperparameters, and model instances. These operations are treated as atomic and require constant time $\mathcal{O}(1)$.

Next, the training procedure is executed by directly applying the augmented Lagrangian, which involves three nested loops: 1) governed by $k\_max\_iter$, 2) constrained by the range of values for $c$, and 3) determined by the number of $epochs$ in the training process. In the worst-case scenario, each loop runs to its maximum limit, and each has linear complexity. Assuming the range for each loop is constant, the time complexity of optimizing the augmented Lagrangian parameters depends solely on the number of data variables in the input dataset, resulting in a complexity of $\mathcal{O}(d)$ per each individual loop, where $d$ represents the number of variables in the observational data. Considering the three nested loops and the parameter optimization step (which takes constant time, $\mathcal{O}(1)$), the overall computational complexity of the augmented Lagrangian is cubic, $\mathcal{O}(d^3)$. 

Inside the augmented Lagrangian, the training algorithm is divided into two stages: causal structure learning and tabular data synthesis, with an additional step for transferring causal knowledge between the stages, which takes constant time $\mathcal{O}(1)$. Both stages utilize stochastic gradient descent (SGD) for optimizing model parameters. Generally, the computational complexity of SGD is $\mathcal{O}(knd)$, where $k$ is the number of epochs, $n$ is the number of samples, and $d$ is the number of variables in $\mathbf{X}$. For DAGAF, both $k$ and $n$ are constant hyperparameters, meaning the optimization complexity depends solely on the number of data attributes in the input. Therefore, the total computational complexity for both stages is linear, $\mathcal{O}(d)$.

The overall time complexity of Algorithm \ref{alg:cap} is given by $\mathcal{O}(d)^3 + 2\mathcal{O}(d)$, which simplifies to $\mathcal{O}(d)^3$ as we focus on the fastest-growing term. This analysis shows that DAGAF has a cubic computational complexity, aligning with results reported for similar algorithms in previous studies \cite{Zheng2018DAGsWN}, \cite{Lachapelle2020GradientBasedND}.

\section{Causal identifiability}
\label{sec:3.1.6}


Our theoretical analysis demonstrates that the DAG model is unique and hence identifiable under the assumptions of the DAGAF framework, which include ANM, LiNGAM, and PNL. This analysis is conducted under the assumption that the data is continuous and follows i.i.d. conditions.\\


\begin{restatable}{proposition}{primepropfive}
\label{prop5}
    Under the Additive Noise Model (ANM), Linear non-Gaussian Acyclic Model (LiNGAM) or Post-Nonlinear Model (PNL) assumption, there exists a unique DAG $\mathcal{G}_\mathcal{A}$ capable of defining the observed joint distribution $P(\mathbf{X})$.  
\end{restatable}

\begin{proof}
    The proof of Proposition \ref{prop5} is available in Appendix \ref{a5}.
\end{proof}

Proposition \ref{prop5} establishes that for a joint distribution $P(\mathbf{X})$ over random variables $\{X_1,...,X_d\}$ generated by a true causal graph $\mathcal{G}_\mathcal{A}$, there exists an identifiable causal graph $G_A$ such that $G_A = \mathcal{G}_\mathcal{A}$,  provided that the causal model follows the ANM, LiNGAM, or PNL assumptions.\\

\noindent In addition, we analyze how the loss terms used to train DAGAF behave under challenging conditions, including non-i.i.d. data, missing values, and discrete variables. 


\subsection{Impact of Non-i.i.d. Conditions}



Now we consider some real-world data case, where the samples $\{\mathbf{X}_1,...,\mathbf{X}_n\}$ are no longer independent (i.e. $\mathbf{X}_i \not\independent \mathbf{X}_j$) and each data point $\mathbf{X}_i$ is drawn from heterogeneous distributions $P_i(\mathbf{X})$. In such settings, the empirical distribution $P^{\prime}(\mathbf{X})$ becomes a biased estimate of the true distribution $P(\mathbf{X})$, impacting the optimization.


We assume that the true and the implicitly generated distributions are defined as $P^{\prime}(\mathbf{X}) = P(\mathbf{X}) + \delta(\mathbf{X})$ and $P^{\prime}_{G_A}(\tilde{\mathbf{X}}) = P_{G_A}(\tilde{\mathbf{X}}) + \delta(\tilde{\mathbf{X}})$, where $\delta(\mathbf{X})$ and $\delta(\tilde{\mathbf{X}})$ capture deviations from the i.i.d. assumptions. 

\subsubsection{Adversarial Loss and Identifiability}

\noindent Under non i.i.d. condition, $\mathcal{L}^{\prime}_{\text{adv}}(\mathbf{X}, \tilde{\mathbf{X}}) = D(P^{\prime}(\mathbf{X})||P_{G_A}(\tilde{\mathbf{X}}))$. The bias $\delta(\mathbf{X})$ affects the gradients of $\mathcal{L}^{\prime}_{\text{adv}}(\mathbf{X}, \tilde{\mathbf{X}})$:
\[
    \nabla_\phi \mathcal{L}^{\prime}_{\text{adv}}(\mathbf{X}, \tilde{\mathbf{X}}) = \nabla_\phi D(P(\mathbf{X})||P_{G_A}(\tilde{\mathbf{X}})) + \nabla_\phi D(\delta(\mathbf{X})||P_{G_A}(\tilde{\mathbf{X}})).
\] The additional term $\nabla_\phi D(\delta(\mathbf{X})||P_{G_A}(\tilde{\mathbf{X}}))$ can destabilize optimization by adding spurious gradient components due to dependencies or heterogeneity, and by amplifying sensitivity to noise in the data.


\subsubsection{MSE Loss and Identifiability}


Under the non-i.i.d. conditions: 
\[
    \mathcal{L}^{\prime}_{\text{MSE}}(\mathbf{X}, \tilde{\mathbf{X}}) = \mathcal{L}_{\text{MSE}}(\mathbf{X}, \tilde{\mathbf{X}}) + \delta(\mathbf{X}).
\]If $\delta(\mathbf{X})$ introduces correlations between samples $\mathbf{X}_i$ and $\mathbf{X}_j$, this violates the independence of the noise terms $\mathcal{Z}_j$. As a result, the non-i.i.d. MSE loss term $\mathcal{L}^{\prime}_{\text{MSE}}(\mathbf{X}, \tilde{\mathbf{X}})$ may incorrectly fit spurious patterns across samples. In turn, the output of $f_j({Pa}_j)$ may no longer capture the true functional relationship.

Furthermore, the gradient of $\mathcal{L}^{\prime}_{\text{MSE}}(\mathbf{X}, \tilde{\mathbf{X}})$ with respect to $\theta$ is:
\[
    \nabla_\theta \mathcal{L}^{\prime}_{\text{MSE}}(\mathbf{X}, \tilde{\mathbf{X}}) = \nabla_\theta \mathcal{L}_{\text{MSE}}(\mathbf{X}, \tilde{\mathbf{X}}) + \nabla_\theta \delta(\mathbf{X}).
\] The additional term $\nabla_\theta \delta(\mathbf{X})$ introduces instability due to spurious gradients from dependencies across samples, and heterogeneity-induced noise in gradients. This instability makes optimization sensitive to the choice of initialization and hyperparameters, thus reducing convergence reliability.


\subsubsection{Kullback-Leibler Divergence Loss and Identifiability}


The empirical estimate of the KLD under non-i.i.d. conditions becomes:
\[
    \mathcal{L}^{\prime}_{\text{KLD}}(\mathbf{X}, \tilde{\mathbf{X}}) = \frac{1}{n} \sum^n_{i=1} \log \frac{P_{G_A}(\tilde{\mathbf{X}}_i)}{P^{\prime}(\mathbf{X}_i)}. 
\] 
Expanding $\mathcal{L}^{\prime}_{\text{KLD}}(\mathbf{X}, \tilde{\mathbf{X}})$ and applying a first-order Taylor expansion $P(\mathbf{X}_i)$, we have 
\[
    \mathcal{L}^{\prime}_{\text{KLD}}(\mathbf{X}, \tilde{\mathbf{X}}) \approx \mathcal{L}_{\text{KLD}}(\mathbf{X}, \tilde{\mathbf{X}}) - \frac{1}{n} \sum^n_{i=1} \frac{\delta(\mathbf{X}_i)}{P(\mathbf{X}_i)}.
\] The term $\frac{\delta(\mathbf{X}_i)}{P(\mathbf{X}_i)}$ introduces bias, particularly when $\delta(\mathbf{X}_i)$ varies significantly across samples. This bias skews the optimization of $P_{G_A}(\tilde{\mathbf{X}})$, which potentially leads to an approximate distribution $P_{G_A}(\tilde{\mathbf{X}})$ that deviates from $P(\mathbf{X})$.



The gradient of the KLD loss under non-i.i.d. conditions is defined as:
\[
    \nabla_\theta \mathcal{L}^{\prime}_{\text{KLD}}(\mathbf{X}, \tilde{\mathbf{X}}) \approx \nabla_\theta \mathcal{L}_{\text{KLD}}(\mathbf{X}, \tilde{\mathbf{X}}) - \int \nabla_\theta P_{G_A}(\tilde{\mathbf{X}}) \frac{\delta(\mathbf{X})}{P(\mathbf{X})}d\mathbf{X}d\tilde{\mathbf{X}}.
\] The additional term $\int \nabla_\theta P_{G_A}(\tilde{\mathbf{X}}) \frac{\delta(\mathbf{X})}{P(\mathbf{X})}d\mathbf{X}d\tilde{\mathbf{X}}$ adds noise to the gradients, reducing the stability of optimization. This may introduce spurious directions in the parameter space, which make convergence to the true distribution $P(\mathbf{X})$ more challenging.


\subsubsection{MMD Loss and Identifiability}

Expanding all instances of $k(.)$, 
we have:

\begin{equation*}
    \begin{aligned}
        k(\mathbf{X}_i, \mathbf{X}_j) &= k(P(\mathbf{X}_i), P(\mathbf{X}_j)) + \Delta_{P(\mathbf{X})}(\mathbf{X}_i, \mathbf{X}_j), \\
        k(\mathbf{X}_i, \tilde{\mathbf{X}}_j) &= k(P(\mathbf{X}_i), P_{G_A}(\tilde{\mathbf{X}}_j)) + \Delta_{P(\mathbf{X}),P_{G_A}(\tilde{\mathbf{X}})}(\mathbf{X}_i, \tilde{\mathbf{X}}_j), \\
        k(\tilde{\mathbf{X}}_i, \tilde{\mathbf{X}}_j) &= k(P_{G_A}(\tilde{\mathbf{X}}_i), P_{G_A}(\tilde{\mathbf{X}}_j) + \Delta_{P_{G_A}(\tilde{\mathbf{X}})}(\tilde{\mathbf{X}}_i, \tilde{\mathbf{X}}_j),
    \end{aligned}
\end{equation*} where $\Delta_{P(\mathbf{X})}(\mathbf{X}_i, \mathbf{X}_j)$, $\Delta_{P(\mathbf{X}),P_{G_A}(\tilde{\mathbf{X}})}(\mathbf{X}_i, \tilde{\mathbf{X}}_j)$ and $\Delta_{P_{G_A}(\tilde{\mathbf{X}})}(\tilde{\mathbf{X}}_i, \tilde{\mathbf{X}}_j)$ represent perturbations due to non-i.i.d. effects. The empirical MMD becomes:
\[
    \mathcal{L}^{\prime}_{\text{MMD}}(\mathbf{X}, \tilde{\mathbf{X}}) \approx \mathcal{L}_{\text{MMD}}(\mathbf{X}, \tilde{\mathbf{X}}) + \Delta,
\] where the non-i.i.d. effect $\Delta$ is defined as follows:
\[
     \Delta = \frac{1}{n}\sum\limits_{i \neq j}^{n}\Delta_{P(\mathbf{X})}(\mathbf{X}_i,\mathbf{X}_j)  - \frac{2}{n}\sum\limits_{i \neq j}^{n}\Delta_{P(\mathbf{X}),P_{G_A}(\tilde{\mathbf{X}})}(\mathbf{X}_i, \tilde{\mathbf{X}}_j) + \frac{1}{n}\sum\limits_{i \neq j}^{n}\Delta_{P_{G_A}(\tilde{\mathbf{X}})}(\tilde{\mathbf{X}}_i, \tilde{\mathbf{X}}_j)
\] The term $\Delta$ introduces bias into the empirical MMD estimate, which may no longer converge to the true population MMD even as $n \to \infty$.

The gradient of $\mathcal{L}^{\prime}_{\text{MMD}}(\mathbf{X}, \tilde{\mathbf{X}})$ with respect to model parameters $\theta$ is:
\[
    \nabla_\theta \mathcal{L}^{\prime}_{\text{MMD}}(\mathbf{X}, \tilde{\mathbf{X}}) = 2 \bigg( \mathbb{E}_{\mathbf{X}, \mathbf{X}^{\prime} \sim P^{\prime}(\mathbf{X})}[\nabla_\theta k(\mathbf{X}, \mathbf{X}^{\prime})] - \mathbb{E}_{\mathbf{X} \sim P^{\prime}(\mathbf{X}), \tilde{\mathbf{X}} \sim P^{\prime}_{G_A}(\tilde{\mathbf{X}})}[\nabla_\theta k(\mathbf{X}, \tilde{\mathbf{X}})] \bigg).
\] The additional perturbations $\Delta_{P(\mathbf{X})}$, $\Delta_{P(\mathbf{X}),P_{G_A}(\tilde{\mathbf{X}})}$ and $\Delta_{P_{G_A}(\tilde{\mathbf{X}})}$ introduce noise into the gradients, potentially destabilizing optimization and making convergence difficult.



\subsection{DAG identifiability in Discrete Variables}



Different DAGs can give rise to the \emph{same} joint distribution in the discrete setting, thereby leading to non-uniqueness in identifying the true DAG $\mathcal{G}_\mathcal{A}$.
For simplicity, consider two DAGs $\mathcal{G}_{1_{\mathcal{A}_1}}$ and $\mathcal{G}_{2_{\mathcal{A}_2}}$ that are structurally different but induce the same joint
distribution. In a \emph{discrete setting}, the symmetry between causal relations often implies that reversing edges or reparameterizing certain relationships leads to the same joint distribution. More formally:

\[
\begin{aligned}
P(X_i \mid \text{Pa}(X_i)) &= P(X_j \mid \text{Pa}(X_j)) \\
\quad &\text{for some} \ (X_i, X_j) \text{ such that } X_j \in \text{Pa}(X_i) \text{ or } X_i \in \text{Pa}(X_j).
\end{aligned}
\]

This symmetry implies that the conditional distributions from both DAG are equal. Thus, the
\emph{identifiability of the DAG} is lost in the discrete setting due to the \emph{equivalence} of the conditional distributions, even though the underlying structural graph may differ. \\




\subsection{Impact of Missing Data} 


Missing data in real-world datasets can arise from different mechanisms. If data is Missing Completely at Random, the missingness is unrelated to any variables, reducing sample size but preserving identifiability with sufficient data. Missing at Random occurs when missingness depends only on observed variables, potentially introducing bias in independence tests but still allowing DAG discovery with robust imputation. Missing Not at Random is the most problematic, as missingness depends on unobserved factors, making the dataset unrepresentative of the true causal structure.

As the identifiability of the true DAG $\mathcal{G}_\mathcal{A}$ relies heavily on correctly testing conditional independence relationships (e.g., $\mathcal{Z}_j \independent {Pa}_j$ in the PNL model), missing data reduces the statistical power of these tests. Missing large portions of data may lead to unreliable or incorrect conditional independence tests. Spurious dependencies or independencies may arise due to imputation strategies or biased sampling. The ANM, LiGAM and PNL model assume that the noise term $\mathcal{Z}_j$ is independent of its parents ($\mathcal{Z}_j \independent {Pa}_j$). Missing data can obscure or distort observed relationships, making it difficult to separate noise from modeled contributions.




In addition, the functional forms \( f_{j} \) (nonlinear for ANM, linear for LiNGAM) and \( g_j \) (nonlinear for PNL) are assumed to be known or learnable. 
However, the data incompleteness characteristic often associated with real-world data violates this assumption. More specifically, missing data biases noise estimates \( \mathcal{Z}_j \), affecting residual independence. In the LiNGAM case, non-Gaussian noise becomes harder to test.


Identifiability relies on correctly estimating marginal distributions. Missing data distorts these estimates, especially when parent variables or structural nodes are disproportionately missing.
\color{black}

\section{Experimental Results}
\label{sec:5}

We conduct a range of experiments on the proposed general framework for causal structure learning using various datasets that include continuous and discrete data types to assess the following aspects: \\

\begin{itemize}
    \item Structure learning accuracy, which assesses the effectiveness of modeling the relationships between features in observational data. \\
    \item Synthetic data quality, which investigates the quality of the data produced from the learned generative process. \\
    \item Ablation study and sensitivity analysis to assess the configuration of the loss terms and the hyper-parameter settings for the training. - for more information, the reader is referred to Appendices \ref{ab} and \ref{ac}.  \\
\end{itemize} In this section, we outline the configurations for the causal discovery and data quality experiments, and present the results along with the metrics employed for their evaluation. 


For evaluating structure learning, our model is compared with leading DAG-learning methods, including DAG-WGAN \cite{Petkov2022DAGWGANCS}, DAG-WGAN+ \cite{Petkov2023EfficientGA}, DAG-Notears-MLP \cite{Zheng2019LearningSN}, Dag-Notears \cite{Zheng2018DAGsWN}, DAG-GNN \cite{Yu2019DAGGNNDS}, GraN-DAG \cite{Lachapelle2020GradientBasedND}, GAE \cite{Ng2019AGA},  CAREFL \cite{Khemakhem2020CausalAF}, DAG-NF \cite{Wehenkel2020GraphicalNF}, DCRL \cite{Mamaghan2024DiffusionBasedCR} and VI-DP-DAG \cite{Charpentier2022DifferentiableDS}. The metric used throughout all experiments to measure the quality of the discovered causality is the Structural Hamming Distance (SHD) \cite{Jongh2009ACO}. We selected SHD because it integrates several individual metrics, including True Positive Rate (TPR), False Discovery Rate (FDR), and False Positive Rate (FPR). It is important to acknowledge that the set of metrics $\text{SHD} = \{\text{TPR}, \text{FDR}, \text{FPR}\}$ used in this study is not the only approach to evaluating the accuracy of the learned structures. Other metrics, such as Area Under Curve (AUC) and Area Over Curve (AOC), can also be employed.

We also analyze the quality of the synthetic data produced by DAGAF. In particular, we conduct various tests to examine the statistical properties of $\tilde{\mathbf{X}}$. We evaluate the similarity between $P(\mathbf{X})$ and $P_{G_A}(\tilde{\mathbf{X}})$ using boxplot analysis and marginal distributions. Additionally, we calculate the correlation matrices for both $\chi$ and $\tilde{\chi}$ to explore the interdependencies among their covariates.

\subsection{Continuous data}

We conduct tests on continuous data types using simulated data produced from predefined structural equations and Directed Acyclic Graph (DAG) structures. Specifically, we construct an Erdos-Renyi \cite{Erds1959OnRG} causal graph with an expected node degree of 3, which serves as the ground-truth DAG $\mathcal{G}_\mathcal{A}$ and can be represented by a weighted adjacency matrix $A$. Afterwards, we generate 5000 observational data samples for each test by utilizing different equations (namely linear: $\tilde{X} := A^TX + \mathcal{Z}$, non-linear-1: $\tilde{X} := Acos(X + 1) + \mathcal{Z}$, non-linear-2: $\tilde{X} := 2sin(A(X + 0.5)) + A(X + 0.5) + \mathcal{Z}$, post-non-linear-1: $\tilde{X} := sinh(Acos(X + 1) + \mathcal{Z})$, and post-non-linear-2: $\tilde{X} := tanh(2sin(A(X + 0.5)) + A(X + 0.5) + \mathcal{Z})$).  
These structural equations have been widely used in numerous papers in DAG learning, including the DAG-GNN model \cite{Yu2019DAGGNNDS}, Gran-DAG \cite{Lachapelle2020GradientBasedND}, GAE \cite{Ng2019AGA}, DAG-WGAN \cite{Petkov2022DAGWGANCS}, DAG-WGAN+ \cite{Petkov2023EfficientGA} and Notears-MLP \cite{Zheng2019LearningSN} - to name but a few. The application of these popular equations allow us to perform a comprehensive and robust comparison against other leading models in the field. The final two equations are modifications of the second and third ones designed to provide suitable test cases for experiments involving the PNL assumption. Ensuring the acyclicity of $\mathcal{G}_\mathcal{A}$ and satisfying the causal model assumptions outlined in Section \ref{sec:2}, with the given above equations, enables us to generate i.i.d. samples that are appropriate for causal structure learning under the faithfulness condition. \\ 

\begin{remark}
Although the list of equations provided in this section serves as a good collection of test cases for the continuous data experiments, it is not exhaustive. Other equations can be used as well.\\
\end{remark}

Our work follows the same methodology used in most other state-of-the-art DAG learning models, such as DAG-GNN, GraN-DAG, DAG-Notears and GAE among others, where the process of splitting data into training and validation sets is not as commonly applied as in traditional machine learning. Train-test splitting or cross-validation is typically used in predictive modeling tasks, but causal structure identification is focused on structural constraints and conditional independencies rather than predictive accuracy. Since causal relationships are structural, they are generally assumed to hold throughout the dataset, and therefore, partitioning the data may not provide significant additional benefit in terms of discovering the structure.

To evaluate the scalability of the model, we perform experiments with datasets that have 10, 20, 50, and 100 columns. To account for sample randomness and ensure fairness, each experiment is repeated five times, and the average Structural Hamming Distance (SHD) is reported. The results are shown in Tables \ref{tab:table1}, \ref{tab:table2}, \ref{tab:table3}, \ref{tab:table4} and \ref{tab:table5}.

\begin{table}[!ht]
\centering
\caption{DAG structures recovered from linear data}
\label{tab:table1}
\begin{tabular}{@{}ccccc@{}}
\toprule
\multirow{2}{*}{Model} & \multicolumn{4}{c}{SHD (5000 linear samples)}                       \\ \cmidrule(l){2-5} 
                       & d=10          & d=20           & d=50            & d=100            \\ \midrule
DAG-Notears            & 8.6 $\pm$ 7.2 & 13.8 $\pm$ 9.6 & 41.8 $\pm$ 29.4 & 102.8 $\pm$ 53.2 \\
DAG-Notears-MLP        & 4.6 $\pm$ 4.3 & 7.6 $\pm$ 6.3  & 29.6 $\pm$ 18.5 & 74 $\pm$ 30.6    \\
DAG-GNN                & 6 $\pm$ 6.9   & 11.4 $\pm$ 8.2 & 33.6 $\pm$ 21.2 & 85.4 $\pm$ 46.4  \\
GAE                    & 5.5 $\pm$ 4.9 & 10.3 $\pm$ 7.2 & 31.3 $\pm$ 13.8 & 80.2 $\pm$ 24.6  \\
GraN-DAG               & 3.4 $\pm$ 5.2 & 6.4 $\pm$ 7.5  & 25.2 $\pm$ 14.6 & 68.4 $\pm$ 25.8  \\
CAREFL                 & 2.7 $\pm$ 4.8 & 5.9 $\pm$ 7.1  & 24.9 $\pm$ 14.1 & 66.9 $\pm$ 24.7  \\
DAG-NF                 & 2.4 $\pm$ 4.6 & 5.2 $\pm$ 6.9  & 23.1 $\pm$ 13.4 & 64.2 $\pm$ 24.3  \\
VI-DP-DAG              & 2.1 $\pm$ 4.5 & 4.5 $\pm$ 6.7  & 22.4 $\pm$ 12.7 & 63.7 $\pm$ 23.5  \\
DCRL                   & 1.8 $\pm$ 2.7 & 3.1 $\pm$ 4.8  & 18.7 $\pm$ 11.9 & 53.3 $\pm$ 21.9  \\
DAG-WGAN               & 5.2 $\pm$ 3.8 & 9.2 $\pm$ 5.7  & 19.6 $\pm$ 12.3 & 58.6 $\pm$ 22.7  \\
DAG-WGAN+              & 3.7 $\pm$ 3.1 & 5.6 $\pm$ 4.9  & 17.2 $\pm$ 10.5 & 49.1 $\pm$ 20.1  \\
DAGAF & \textbf{1.4 $\pm$ 2.3} & \textbf{2 $\pm$ 4.4} & \textbf{16.4 $\pm$ 9.8} & \textbf{38.8 $\pm$ 18.3} \\ \bottomrule
\end{tabular}%
\end{table}

\begin{table}[!ht]
\centering
\caption{DAG structures recovered from non-linear-1 data}
\label{tab:table2}
\begin{tabular}{@{}ccccc@{}}
\toprule
\multirow{2}{*}{Model} & \multicolumn{4}{c}{SHD (5000 non-linear-1 samples)}                   \\ \cmidrule(l){2-5} 
                       & d=10           & d=20            & d=50            & d=100            \\ \midrule
DAG-Notears            & 11.4 $\pm$ 4.5 & 28.2 $\pm$ 10.2 & 55 $\pm$ 23.1   & 105.6 $\pm$ 48.3 \\
DAG-Notears-MLP        & 5.2 $\pm$ 1.8  & 15.4 $\pm$ 4.6  & 43.8 $\pm$ 15.4 & 86.2 $\pm$ 29.8  \\
DAG-GNN                & 9.2 $\pm$ 3.3  & 23.4 $\pm$ 8.4  & 50.2 $\pm$ 19.5 & 98.6 $\pm$ 37.6  \\
GAE                    & 8.6 $\pm$ 2.2  & 20 $\pm$ 5.7    & 47.5 $\pm$ 10.2 & 92.3 $\pm$ 18.9  \\
GraN-DAG               & 4 $\pm$ 2.4    & 11.2 $\pm$ 6.5  & 36.4 $\pm$ 11.9 & 72.8 $\pm$ 21.7  \\
CAREFL                 & 3.8 $\pm$ 2.2  & 10.9 $\pm$ 6.2  & 34.1 $\pm$ 11.2 & 71.7 $\pm$ 19.1  \\
DAG-NF                 & 3.4 $\pm$ 2.1  & 10.4 $\pm$ 5.6  & 31.6 $\pm$ 10.7 & 69.5 $\pm$ 17.3  \\
VI-DP-DAG              & 3.1 $\pm$ 2    & 9.8 $\pm$ 5.1   & 28.7 $\pm$ 9.3  & 68.1 $\pm$ 16.5  \\
DCRL                   & 2.9 $\pm$ 1.7  & 7.5 $\pm$ 4     & 24.3 $\pm$ 7.8  & 61.4 $\pm$ 14.9  \\
DAG-WGAN               & 6.4 $\pm$ 1.4  & 18.6 $\pm$ 3.7  & 22 $\pm$ 8.6    & 64.6 $\pm$ 15.2  \\
DAG-WGAN+              & 4.9 $\pm$ 1.2  & 14.2 $\pm$ 3.3  & 20.5 $\pm$ 6.9  & 57.1 $\pm$ 14.5  \\
DAGAF & \textbf{2.6 $\pm$ 1} & \textbf{5.2 $\pm$ 2.8} & \textbf{18.8 $\pm$ 6.2} & \textbf{50.2 $\pm$ 13.4} \\ \bottomrule
\end{tabular}
\end{table}

\begin{table}[!ht]
\centering
\caption{DAG structures recovered from non-linear-2 data}
\label{tab:table3}
\begin{tabular}{@{}ccccc@{}}
\toprule
\multirow{2}{*}{Model} & \multicolumn{4}{c}{SHD (5000 non-linear-2 samples)}                  \\ \cmidrule(l){2-5} 
                       & d=10           & d=20           & d=50            & d=100            \\ \midrule
DAG-Notears            & 10.4 $\pm$ 3.9 & 22.4 $\pm$ 8.1 & 47.6 $\pm$ 21.2 & 112.8 $\pm$ 57.8 \\
DAG-Notears-MLP        & 5.4 $\pm$ 1.5  & 13.8 $\pm$ 4.3 & 30.4 $\pm$ 15.7 & 85.6 $\pm$ 35.6  \\
DAG-GNN                & 8.4 $\pm$ 3.2  & 19.2 $\pm$ 7.7 & 36.2 $\pm$ 18.6 & 91.8 $\pm$ 49.3  \\
GAE                    & 7.3 $\pm$ 1.8  & 17.4 $\pm$ 5.1 & 33.7 $\pm$ 13.7 & 88.4 $\pm$ 26.6  \\
GraN-DAG               & 4.2 $\pm$ 2.1  & 11.6 $\pm$ 5.6 & 25.2 $\pm$ 14.5 & 71.6 $\pm$ 29.7  \\
CAREFL                 & 3.8 $\pm$ 1.8  & 10.5 $\pm$ 5.3 & 24.8 $\pm$ 13.8 & 69.9 $\pm$ 26.1  \\
DAG-NF                 & 3.3 $\pm$ 1.7  & 9.7 $\pm$ 4.9  & 24.3 $\pm$ 13.1 & 68.1 $\pm$ 24.3  \\
VI-DP-DAG              & 2.8 $\pm$ 1.6  & 9.3 $\pm$ 4.7  & 23.8 $\pm$ 13.3 & 67.3 $\pm$ 23.8  \\
DCRL                   & 2.2 $\pm$ 1.3  & 7.1 $\pm$ 2.9  & 15.1 $\pm$ 9.4  & 59.5 $\pm$ 17.2  \\
DAG-WGAN               & 6.6 $\pm$ 1.2  & 15.2 $\pm$ 3.4 & 22.6 $\pm$ 12.9 & 64.2 $\pm$ 21.5  \\
DAG-WGAN+              & 5.1 $\pm$ 1.1  & 12.3 $\pm$ 2.5 & 17.5 $\pm$ 10.2 & 56.7 $\pm$ 18.4  \\
DAGAF & \textbf{1.4 $\pm$ 0.9} & \textbf{5.8 $\pm$ 2.2} & \textbf{14.2 $\pm$ 8.3} & \textbf{51.8 $\pm$ 16.2} \\ \bottomrule
\end{tabular}
\end{table}

\begin{table}[!ht]
\centering
\caption{DAG structures recovered from post-non-linear-1 data}
\label{tab:table4}
\begin{tabular}{cccccc}
\toprule
\multirow{2}{*}{Model} & \multicolumn{4}{c}{SHD (5000 post-non-linear-1 samples)} \\ \cmidrule(l){2-5} 
        & d=10 & d=20 & d=50 & d=100 \\ \midrule
DAG-GNN & 13.7 $\pm$ 9.2 & 21.7 $\pm$ 10.4 & 63.7 $\pm$ 31.2 & 118.6 $\pm$ 50.1 \\
GAE & 12.3 $\pm$ 8.1 & 19.1 $\pm$ 8.8 & 56.2 $\pm$ 24.6 & 101.3 $\pm$ 37.4 \\
CAREFL & 11.8 $\pm$ 6.4 & 18.5 $\pm$ 7.9 & 52.1 $\pm$ 22.8 & 97.2 $\pm$ 34.9 \\
DAG-NF & 11.2 $\pm$ 5.3 & 16.2 $\pm$ 6.1 & 47.3 $\pm$ 19.5 & 92.5 $\pm$ 31.3 \\
DAG-WGAN & 10.5 $\pm$ 4.7 & 15.6 $\pm$ 5.8 & 44.5 $\pm$ 17.7 & 88.7 $\pm$ 29.6 \\ 
DAG-WGAN+ & 8.4 $\pm$ 3.3 & 12.8 $\pm$ 4.3 & 32.8 $\pm$ 13.6 & 66.1 $\pm$ 21.2 \\ 
DAGAF & \textbf{5.6 $\pm$ 2.5} & \textbf{7.3 $\pm$ 3.2} & \textbf{25.4 $\pm$ 11.3} & \textbf{52.4 $\pm$ 15.7} \\ \bottomrule
\end{tabular}%
\end{table}

\begin{table}[!ht]
\centering
\caption{DAG structures recovered from post-non-linear-2 data}
\label{tab:table5}
\begin{tabular}{cccccc}
\toprule
\multirow{2}{*}{Model} & \multicolumn{4}{c}{SHD (5000 post-non-linear-2 samples)} \\ \cmidrule(l){2-5} 
        & d=10 & d=20 & d=50 & d=100 \\ \midrule
DAG-GNN & 10.8 $\pm$ 8.7 & 16.1 $\pm$ 11.9 & 37.1 $\pm$ 30.3 & 128.3 $\pm$ 48.2 \\
GAE & 9.1 $\pm$ 6.3 & 14.3 $\pm$ 9.5 & 31.5 $\pm$ 24.8 & 105.7 $\pm$ 34.4 \\
CAREFL & 8.3 $\pm$ 5.8 & 13.5 $\pm$ 8.3 & 29.8 $\pm$ 22.4 & 92.1 $\pm$ 32.3 \\
DAG-NF & 7.7 $\pm$ 5.5 & 12.8 $\pm$ 7.4 & 28.4 $\pm$ 21.7 & 84.8 $\pm$ 28.5 \\
DAG-WGAN & 7.2 $\pm$ 5.2 & 11.4 $\pm$ 6.2 & 25.2 $\pm$ 18.6 & 76.5 $\pm$ 27.6 \\ 
DAG-WGAN+ & 4.5 $\pm$ 3.6 & 8.6 $\pm$ 5.1 & 21.7 $\pm$ 12.3 & 69.4 $\pm$ 19.1 \\ 
DAGAF & \textbf{2.9 $\pm$ 2.4} & \textbf{5.7 $\pm$ 3.6} & \textbf{18.6 $\pm$ 10.5} & \textbf{47.2 $\pm$ 14.7} \\ \bottomrule
\end{tabular}%
\end{table}

The results presented in Tables \ref{tab:table1}, \ref{tab:table2}, \ref{tab:table3}, \ref{tab:table4} and \ref{tab:table5} demonstrate that our proposed general framework for causal discovery consistently outperforms state-of-the-art DAG-learning methods across all tested scenarios—linear, non-linear-1, non-linear-2, post-nonlinear-1, and post-nonlinear-2—regardless of whether the underlying data-generating process follows LiNGAM, ANM, or PNL assumptions. Notably, the gap in SHD between our model and the others grows further in our favor with the increase in data dimensionality. This observation highlights the enhanced performance of our approach for DAG-learning in datasets with a large number of variables. It is also worth mentioning that, according to our results, DAGAF surpasses both traditional models in the field, including Notears, GAE, DAG-GNN, and GraN-DAG, as well as more recent approaches like DAG-WGAN(+), CAREFL, DAG-NF, DCRL and VI-DP-DAG, demonstrating the superiority of our model. 

\subsection{Benchmark experiments}

In our experiments, we also included discrete datasets as part of an empirical study to demonstrate how our framework performs on such data. However, 
from our theoretical analysis presented in Section \ref{sec:3.1.6}, we recognize that identifiability issues arise when applying our method to discrete datasets. 

Specifically, we obtained the Child, Alarm, Hailfinder, and Pathfinder benchmark datasets, with their ground truths, from the Bayesian Network Repository  \url{https://www.bnlearn.com/bnrepository}. These datasets are specifically organized to facilitate scalability testing and enable a fair comparison with state-of-the-art methods. We evaluated our model against DAG-GNN and both versions of DAG-WGAN, with the results presented in Table \ref{tab:table6}.

\begin{table}[!ht]
\centering
\caption{DAG structures recovered from benchmark data}
\label{tab:table6}
\begin{tabular}{@{}cccccc@{}}
\toprule
\multirow{2}{*}{Datasets} & \multirow{2}{*}{Nodes} & \multicolumn{3}{c}{SHD}              \\ \cmidrule(l){3-6} 
                          &                        & DAG-WGAN    & DAG-WGAN+ & DAG-GNN & DAGAF  \\ \midrule
Child                     & 20                     & 20          & 19 & 30      & \textbf{17}  \\
Alarm                     & 37                     & 36 & \textbf{35} & 55      & 43           \\
Hailfinder                & 56                     & 73 & 66          & 71      & \textbf{63}  \\
Pathfinder                & 109                    & 196 & 194          & 218     & \textbf{181} \\ \bottomrule
\end{tabular}%
\end{table}

According to the benchmark experiment results shown in Table \ref{tab:table6}, our method significantly outperforms DAG-GNN across all four datasets (Child, Alarm, Hilfinder, and Pathfinder). Additionally, both DAG-WGAN and its improved version, DAG-WGAN+, deliver inferior results compared to our framework on three out of the four datasets. Similar outcomes are observed in experiments with continuous datasets, where the SHD gap between our method and the others widens as the number of data variables increases. 

\subsection{Real data experiments}
\label{sec:4.3}

While our experiments with simulated data show the ability of DAGAF to generate decent results, they are not entirely conclusive, as simulations differ from real-world scenarios. To address this issue, we conducted experiments using a well-known real-world dataset called Sachs \cite{Sachs-et-al:scheme}, which is widely recognized in the research community. This dataset comprises 7466 samples across 11 columns, with an estimated ground truth containing 20 edges. Additionally, our approach assumed both ANM and PNL during this test and compared the SHD produced by these FCM to determine whether the post-nonlinear model is superior when applied to real-world data. The results are presented in Table \ref{tab:table7}.

\begin{table}[!ht]
\centering
\caption{DAG structures recovered from real data}
\label{tab:table7}
\begin{tabular}{@{}cc@{}}
\toprule
\multirow{2}{*}{Model} & Sachs Dataset \\ \cmidrule(l){2-2} 
                       & SHD           \\ \midrule
DAG-WGAN               & 17            \\
DAG-WGAN+              & 15            \\
DAG-NF                 & 15            \\
DAG-GNN                & 25            \\
GAE                    & 20            \\
GraN-DAG               & 17            \\
VI-DP-DAG              & 16            \\
DAGAF            & ANM \textbf{9} / PNL \textbf{8}    \\ \bottomrule
\end{tabular}%
\end{table}


The experiment with the Sachs dataset shows that our method can also accurately discover DAG structures from real data. As indicated in Table \ref{tab:table7}, our framework significantly outperforms all other state-of-the-art algorithms involved in the study. Additionally, the empirical evidence suggests that the PNL assumption enables our approach to learn a more precise causal structure approximation compared to the application of other identifiable causal models. 

\subsection{Synthetic data quality}
\label{ref:sec4.4}

In this work, we have advocated for the superiority of our method over current state-of-the-art models by combining causality learning with synthetic data generation. To further support this claim, we compare the features (d=10) from two tabular datasets of simulation data (one based on the ANM and the other on the PNL assumption) with the features generated by our approach. We consider the special case where our model achieves an SHD of 0 on the simulation data, as this would result in the highest quality samples due to the complete knowledge of causal mechanisms in the generative process.

We conduct the following analyses to compare the real and synthetic data: computing the correlation matrices, visualizing the joint and marginal distributions, investigating distributional consistency with Principal Component Analysis (PCA) \cite{Jolliffe2016PrincipalCA} and performing machine learning regression to compare the feature importance in both datasets. Our findings demonstrate that the synthetic samples generated by the proposed framework accurately replicate the correlations (Figure \ref{fig:1}) along with the joint and marginal distributions of the features present in the observational data (Figure \ref{fig:2}).  Furthermore, the generated data captures the underlying patterns and structure of the original data (Figure \ref{fig:9}),  and contains enough predictive information to support regression tasks (Figure  \ref{fig:6}). We present only a few examples of each analysis in this section; additional results can be found in Appendix \ref{ad}.


\begin{figure*}[!ht]
\begin{center}
    \begin{subfigure}[t]{0.48\textwidth}
        \includegraphics[width=\textwidth]{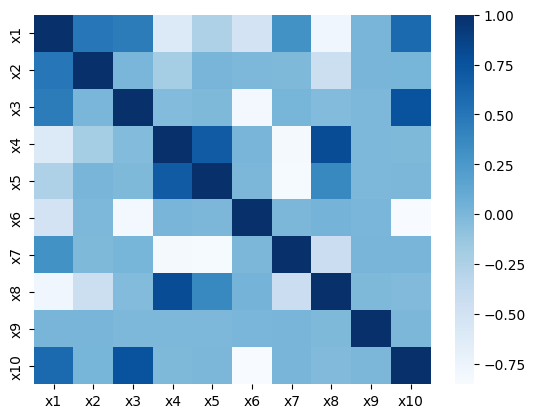} 
    \end{subfigure}
     \begin{subfigure}[t]{0.48\textwidth}
        \includegraphics[width=\textwidth]{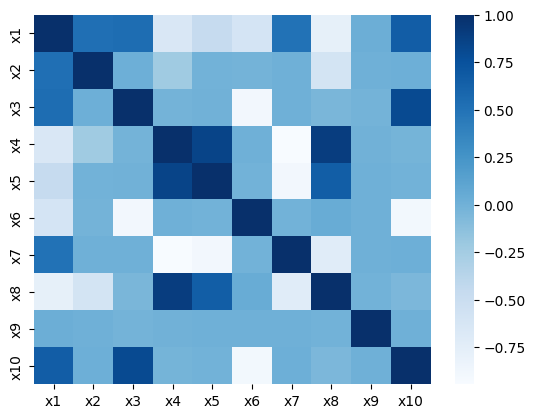} 
    \end{subfigure}
    \begin{subfigure}[t]{0.48\textwidth}
        \includegraphics[width=\textwidth]{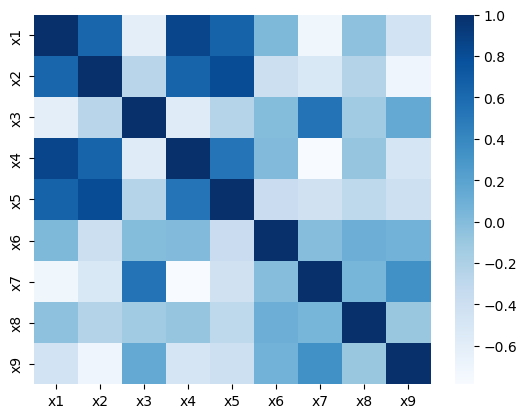} 
    \end{subfigure}
     \begin{subfigure}[t]{0.48\textwidth}
        \includegraphics[width=\textwidth]{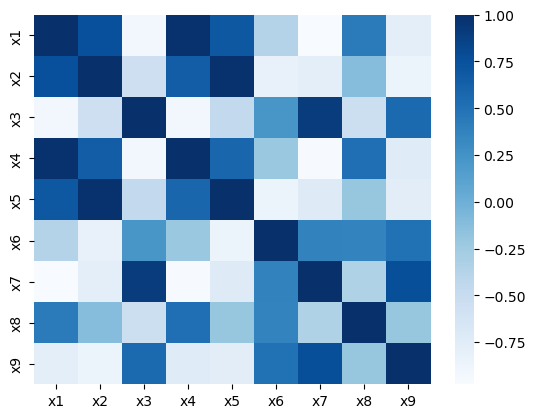} 
    \end{subfigure}
\end{center}
\caption{
 Comparison of the correlation matrices for real (left) and synthetic (right) features reveals
that the statistical correlations across the feature space for both real and synthetic data are nearly
identical, in both the ANM (first row) and the PNL (second row) case. 
}
\label{fig:1}
\end{figure*}

\begin{figure*}[!ht]
\begin{center}
\begin{subfigure}[t]{0.49\textwidth}
        \includegraphics[width=\textwidth]{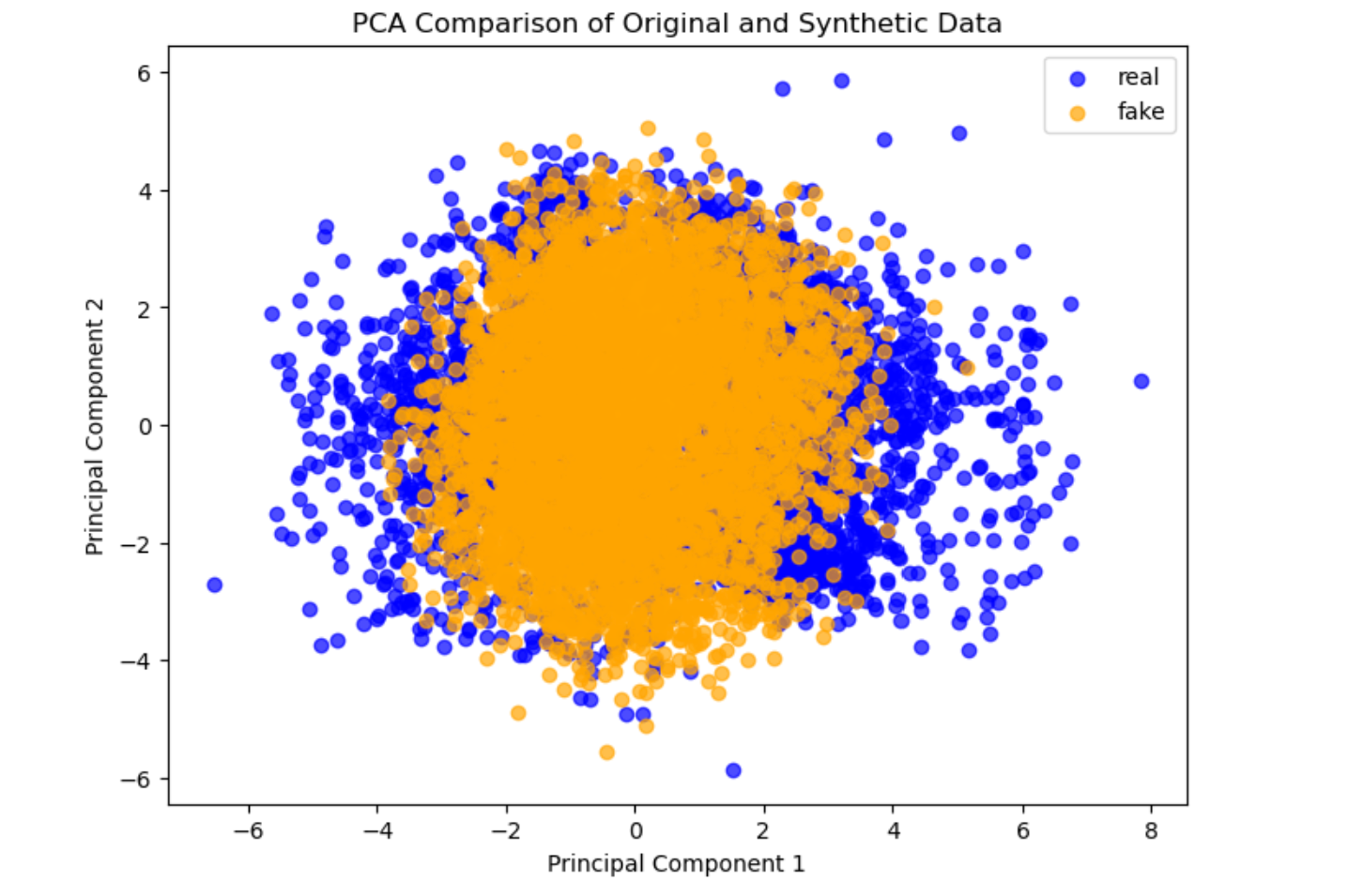} 
    \end{subfigure}
    \begin{subfigure}[t]{0.49\textwidth}
        \includegraphics[width=\textwidth]{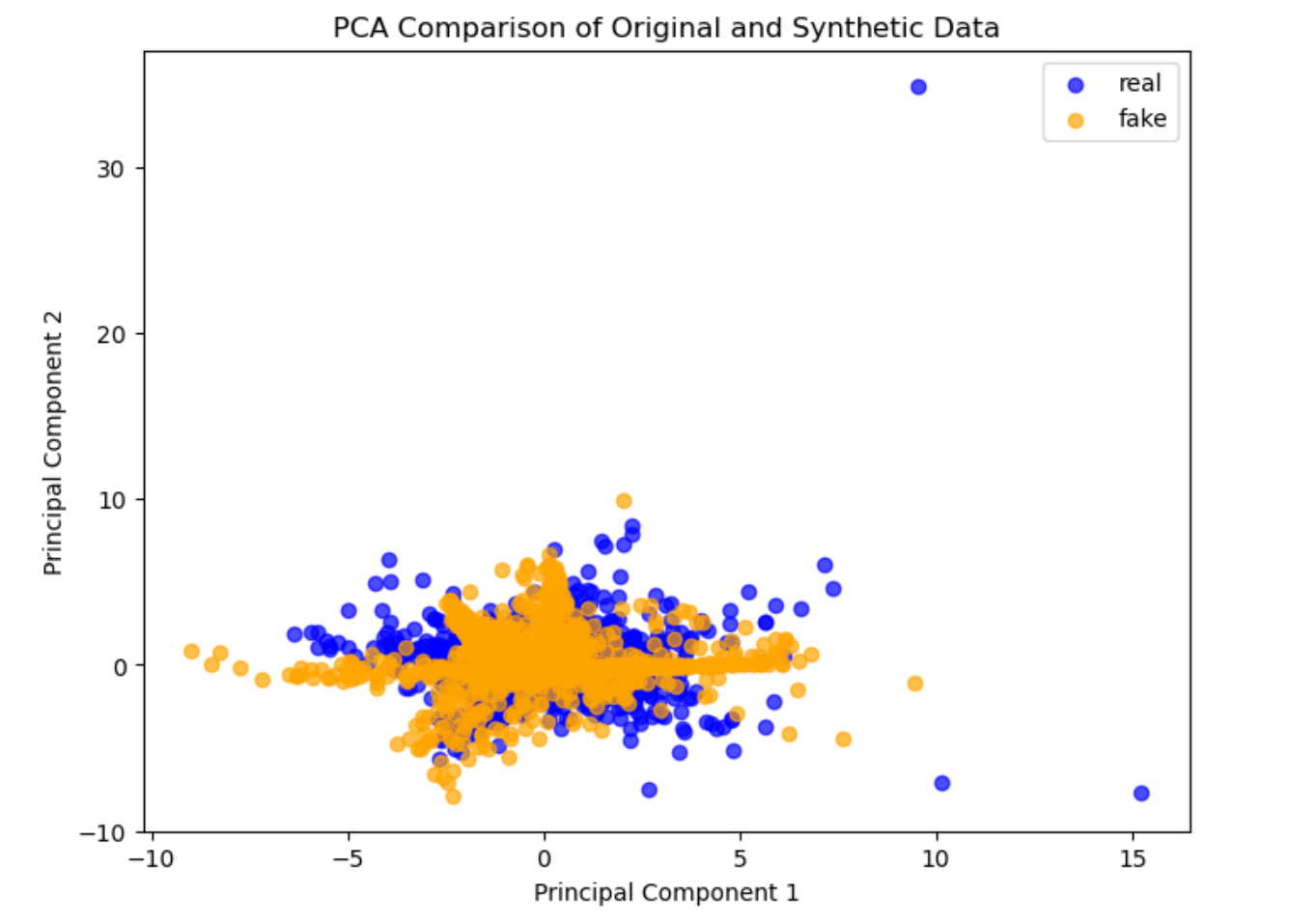} 
    \end{subfigure}
\end{center}
\caption{ Principal Component Analysis (PCA) between the original and synthetic samples for both the ANM (left) and the PNL (right) case. We observe both the input and the synthetic samples have similar clusters and outliers. The results indicate that the implicitly generated distribution resembles the original distribution in both mean and standard deviation, making them indistinguishable from each other.
}
\label{fig:9}
\end{figure*}

\begin{figure*}[!ht]
\begin{center}
    \begin{subfigure}[t]{0.49\textwidth}
        \includegraphics[width=\textwidth]{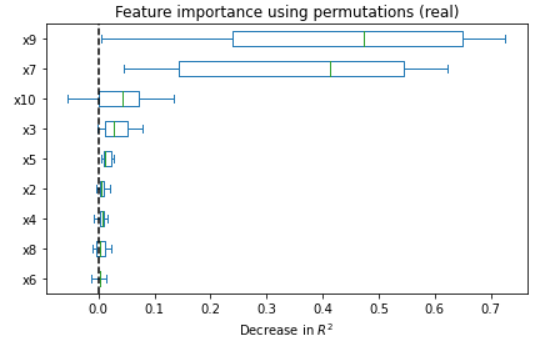} 
    \end{subfigure}
    \begin{subfigure}[t]{0.49\textwidth}
        \includegraphics[width=\textwidth]{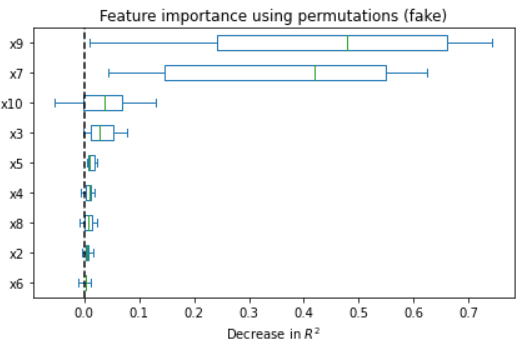} 
    \end{subfigure}
    \begin{subfigure}[t]{0.49\textwidth}
        \includegraphics[width=\textwidth]{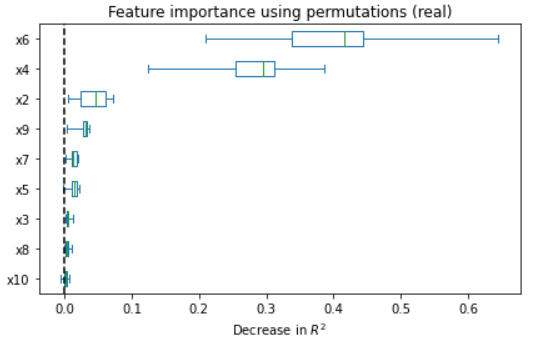} 
    \end{subfigure}
    \begin{subfigure}[t]{0.49\textwidth}
        \includegraphics[width=\textwidth]{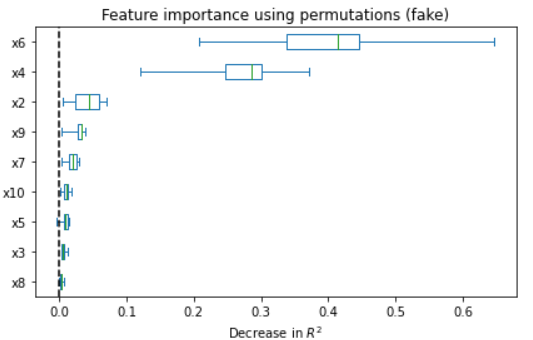} 
    \end{subfigure}
\end{center}
\caption{
 Feature importance comparison between real (left) and synthetic (right) data, in both the ANM (first row) and the PNL (second row) case. The synthetic features with their relevance are indistinguishable from the original ones, allowing for their application in regression tasks.
}

\label{fig:6}
\end{figure*}

\begin{figure*}[!ht]
\begin{center}
    \begin{subfigure}[b]{0.43\textwidth}
        \includegraphics[width=\textwidth]{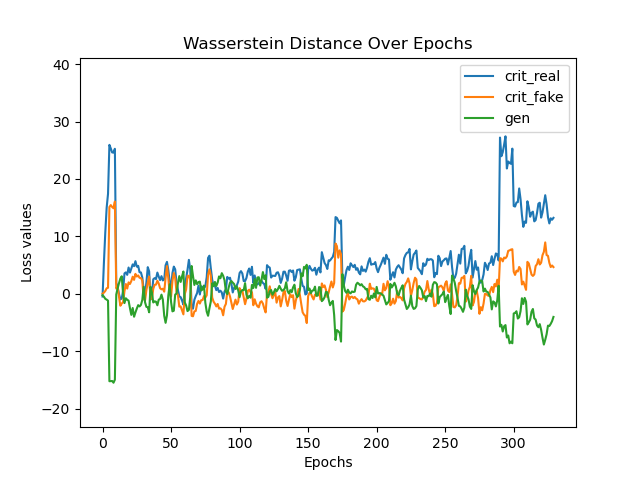} 
    \end{subfigure}
\end{center}
\caption{
Visualizing the Wasserstein distance between the original and synthetic data over the course of the augmented Lagrangian algorithm. The significant discrepancy between the real and the generated samples (165-170 and from 300 epochs onward) occurs because of fluctuations in the SHD, courtesy of the parameter-tuning for the continuous optimization approach. Conversely, the lowest SHD is detected when the Wasserstein Distance is at its lower conversions (50-150 and 175 - 275 epochs). 
}
\label{fig:5}
\end{figure*}

\begin{figure*}[!ht]
\begin{center}
    \begin{subfigure}[t]{0.35\textwidth}
        \includegraphics[width=\textwidth]{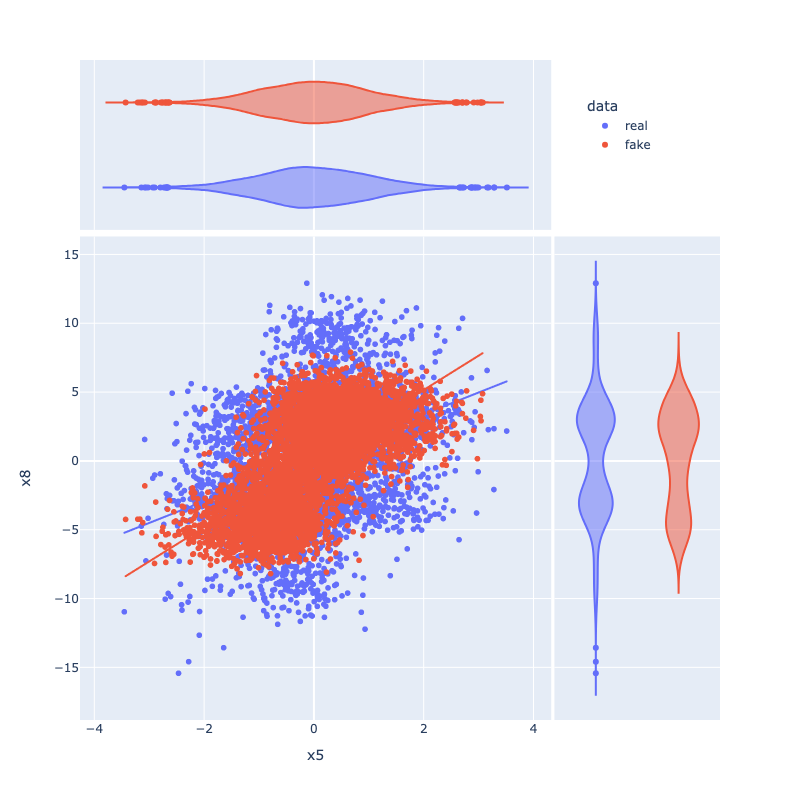} 
    \end{subfigure}
    \begin{subfigure}[t]{0.35\textwidth}
        \includegraphics[width=\textwidth]{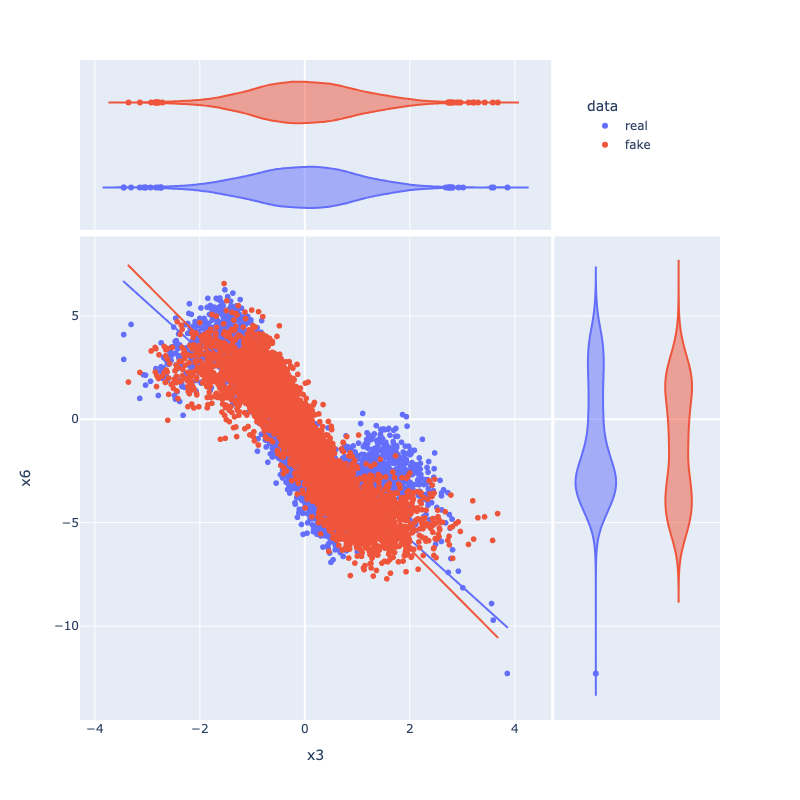} 
    \end{subfigure}
    \begin{subfigure}[b]{0.35\textwidth}
        \includegraphics[width=\textwidth]{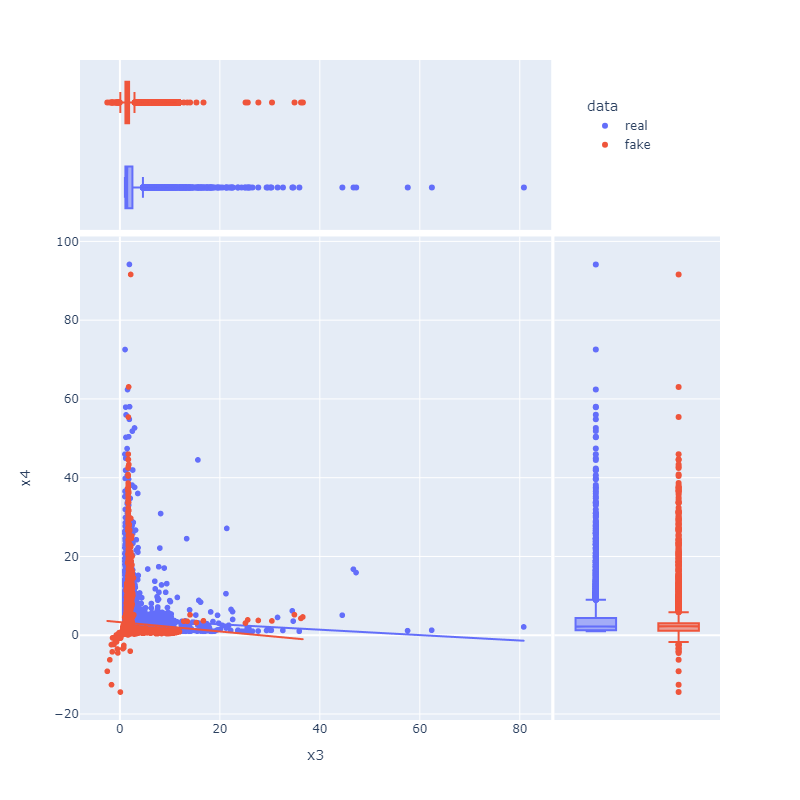} 
    \end{subfigure}
\end{center}
\caption{
 Visualizing the distributions of the real and synthetic features, we plotted x5 against x8 (left), x3 against x6 (right), in the case of ANM, and x3 against x4 for the PNL case. The joint and marginal distributions are accurately modeled with no significant differences between the real and synthetic features.
}
\label{fig:2}
\end{figure*}

\section{Conclusion \& Future Work}

This research introduces a novel framework for multivariate causal structure learning aimed at holistically discovering DAG structures in a dataset to model its generative mechanisms and produce synthetic samples that closely resemble real data. We conducted a theoretical analysis demonstrating that the Wasserstein-1 distance metric can be leveraged for structure learning and explained how the integration of regularization and reconstruction loss terms in our training process can enhance the identification of causal relationships from observational data. Furthermore, we showcased the performance of our approach through extensive experiments, where the method significantly outperformed state-of-the-art DAG-learning techniques. The experimental results demonstrate that our method effectively handles numerical and categorical data types to accurately recover DAG structures under LiNGAM, ANM or PNL assumptions, while generating realistic data samples. The analysis of our results suggests that the Wasserstein distance plays a significant role in enhancing DAG learning. Our findings also indicate a close relationship between the simultaneous generation of diverse high-quality data and the learning of accurate DAG structures, suggesting that the synthesis of realistic data samples is facilitated by the recovery of meaningful variable relationships.


 All results are generated using LiNGAM, ANM or PNL, which are proven to be identifiable \cite{Shimizu2006ALN}, \cite{Hoyer2008NonlinearCD}, \cite{JMLR:v21:19-664}, \cite{Zhang2009OnTI}. However, our experiments have been restricted to these models, which is a limitation. In future work, we plan to explore other identifiable structures, such as generalized linear models, polynomial regression and index models. Furthermore, our tabular data synthesis experiments have also been quite limited, focusing only on analyzing primitive features of datasets. We plan to extend our investigations by comparing the output of DAGAF with other causality-based tabular data generation methods \cite{Breugel2021DECAFGF}, \cite{Rajabi2021TabFairGANFT}, \cite{Wen2021CausalTGANGT}. This comparison will be conducted using more appropriate metrics, such as Cross-Validation Score (CVS) \cite{Stone1976CrossValidatoryCA}, Kolmogorov-Smirnov (KS) test \cite{Simard2011ComputingTT} or Chi-Square test \cite{Williams1950TheCO}, to offer a more comprehensive qualitative analysis of the data generation capabilities of our framework.

In essence, our approach identifies DAG structures by integrating MLE with adversarial loss components and enforcing an acyclicity constraint via an augmented Lagrangian. Consequently, our model exhibits high computational complexity and a complicated loss function. We plan to explore more efficient structure learning methods and adversarial loss training to develop a faster model that relies exclusively on the Wasserstein loss.

The proposed causal learning-based synthetic data generation framework is closely connected to recent advances in generative modeling, including Digital Twins and transformer-based architectures. DAG learning naturally embodies the essence of attention mechanisms by identifying the direct causal parents of each variable, similar to how transformers dynamically weigh relevant dependencies. Moreover, our approach aligns with the principles of Digital Twins, which aim to simulate real-world systems and generate data that accurately reflect their underlying causal structures. This study establishes a unified framework for causal discovery and generative modeling, leveraging adversarial learning, MSE, MMD, and KLD regularization to ensure robust structure learning and high-fidelity synthetic data generation.

Our future work will include several mitigation strategies to address missing data. We will employ data imputation techniques such as mean/mode imputation, multiple imputation, and advanced methods like matrix completion and variational autoencoders (VAEs), while acknowledging that imputation introduces assumptions about missingness that may bias results. Additionally, we will leverage structural information, using partial knowledge of the directed acyclic graph (DAG), such as domain expertise, to help compensate for missing data. Another approach involves explicitly modeling missingness mechanisms by introducing a missingness variable into the DAG to represent whether a specific variable is missing. Moreover, we will also apply causal inference techniques, including latent variable models and specialized methods designed for incomplete data, to ensure robust and accurate analyses.

Finally, as part of our future work, we will examine the flexibility of our framework by experimenting with different combinations of FCM and DGM to identify the optimal configuration for enhancing the output quality of the proposed method and extending its application to time-series data. For example, recently developed concepts such as digital twin layer via multi-attention networks \cite{Poap2024SonarDT}, \cite{KurisummoottilThomas2023CausalSC} can offer exciting avenues for future exploration. This can be achieved through their multi-attention mechanisms, which effectively highlight relevant features while filtering out irrelevant noise and misleading correlations. Their ability to adaptively handle mixed-variable datasets, align higher-order statistics of distributions, and dynamically capture multi-modal dependencies can complement the causal discovery framework presented in this work. Future research could focus on integrating these mechanisms to improve the robustness and scalability of causal discovery and synthetic data generation for complex real-world datasets. Such integration would bridge the gap between foundational theoretical insights and practical applications, addressing challenges like non-i.i.d. data and variable heterogeneity while enabling the creation of robust, high-fidelity synthetic datasets for downstream tasks.

The novel setup will be supported by an extensive study of hyper-parameters to determine their best possible values, resulting in more realistic data samples generated through a more accurately simulated generative process. 

\begin{appendices}

\section{Mathematical Proofs}

This appendix provides the proofs associated with the propositions and theorems found in Section \ref{ref:sec3}.

\subsection{Proof of Proposition~\ref{prop1}}
\label{a1}
\primepropone*

\begin{proof}
    Let $\tilde{\mathbf{X}} \sim P_{G_A}(\tilde{\mathbf{X}})$ denote the distribution generated by a DAG $G_A$. Assume the true data distribution $\mathbf{X} \sim P(\mathbf{X})$ is generated from the ground-truth graph $\mathcal{G}_\mathcal{A}$.
The adversarial loss $\mathcal{L}_{\text{adv}}(\mathbf{X}, \tilde{\mathbf{X}})$ based on the Wasserstein distance $\mathbb{W}_p(P(\mathbf{X}), P_{G_A}(\tilde{\mathbf{X}}))$ is expressed in Equation \eqref{eq2}.
Therefore, minimizing $\mathcal{L}_{\text{adv}}(\mathbf{X}, \tilde{\mathbf{X}})$ aligns $P_{G_A}(\tilde{\mathbf{X}})$ with $P(\mathbf{X})$:

    \[
        P_{G_A}(\tilde{\mathbf{X}}) = P(\mathbf{X}) \implies \mathbb{W}_p(P(\mathbf{X}), P_{G_A}(\tilde{\mathbf{X}})) = 0,
    \] at the global minimum of the distance metric
    \[
        \mathbb{W}_p(P(\mathbf{X}), P_{G_A}(\tilde{\mathbf{X}})) = 0 \implies P_{G_A}(\tilde{\mathbf{X}}) = P(\mathbf{X}).
    \]

    \noindent For $G_A \neq \mathcal{G}_\mathcal{A}$, the generated distribution $P_{G_A}(\tilde{\mathbf{X}})$ cannot match $P(\mathbf{X})$ because the structure $G_A$ is incorrect:
    \[
        \mathbb{W}_p(P(\mathbf{X}), P_{G_A}(\tilde{\mathbf{X}})) > 0.
    \]    


        
    
    \noindent Therefore, minimizing $\mathcal{L}_{\text{adv}}(\mathbf{X}, \tilde{\mathbf{X}})$ aligns $P_{G_A}(\tilde{\mathbf{X}})$ with $P(\mathbf{X})$, and the identifiability assumption guarantees that this occurs only when $G_A = \mathcal{G}_\mathcal{A}$, thus concluding the proof.
\end{proof}

\subsection{Proof of Proposition~\ref{prop2}}
\label{a2}
\primeproptwo*

\begin{proof}

    From the definition of \( \mathcal{L}_{\text{MSE}}(\mathbf{X}, \tilde{\mathbf{X}}) \), it is minimized if and only if:
    \[
        \|\mathbf{X}_i - \tilde{\mathbf{X}}_i\|^2 = 0, \quad \forall \mathbf{X}_i \in P(\mathbf{X}), \forall \tilde{\mathbf{X}}_i \in P_{G_A}(\tilde{\mathbf{X}}), \quad \forall i \in \{1, \dots, n\},
    \] which implies $\mathbf{X}_i = \tilde{\mathbf{X}}_i, \quad \forall i \in \{1, \dots, n\}$.\\

\noindent The gradient of $\mathcal{L}_{\text{MSE}}(\mathbf{X}, \tilde{\mathbf{X}})$ with respect to the model parameters $\theta$ (which define $G_A$) is given by: 

    \[
        \nabla_{\theta} \mathcal{L}_{\text{MSE}}(\mathbf{X}, \tilde{\mathbf{X}}) = \frac{1}{n} \sum_{i=1}^n 2 \cdot ||\mathbf{X}_i - \tilde{\mathbf{X}}_i|| \cdot \nabla_{\theta} \tilde{\mathbf{X}}_i.
    \]

    \noindent As the model predictions $\tilde{\mathbf{X}}_i$ approach the true data $\mathbf{X}_i$ the residual distance $\|\mathbf{X}_i - \tilde{\mathbf{X}}_i\|$ becomes smaller:

    \[
        \|\mathbf{X}_i - \tilde{\mathbf{X}}_i\| \to 0 \quad \implies \quad \nabla_{\theta} \mathcal{L}_{\text{MSE}}(\mathbf{X}, \tilde{\mathbf{X}}) \to 0.
    \] This behavior arises because the residual distance $\|\mathbf{X}_i - \tilde{\mathbf{X}}_i\|$ directly scales the gradient. As $\tilde{\mathbf{X}}_i$ aligns with $\mathbf{X}_i$, the gradient magnitude decreases, reducing the size of updates during optimization.
    Therefore, the MSE loss offers optimization stability by smooth gradients. By steady convergence as $\tilde{\mathbf{X}}_i \to \mathbf{X}_i$, preventing oscillatory behavior, thus concluding the proof.

\end{proof}

\subsection{Proof of Proposition~\ref{prop3}}
\label{a3}
\primepropthree*

\begin{proof}


\noindent This term is used to ensure that the residual noise $\mathcal{Z}_j$ conditioned on ${Pa}_j$ is Gaussian. The residual $\mathcal{Z}_j$ can be expressed as 
    $\mathcal{Z}_j = X_j - f_j({Pa}_j)$.
By minimizing $\mathcal{L}_{\text{KLD}}(\mathbf{X}, \tilde{\mathbf{X}})$, the model is encouraged to fit $f_j$ such that $\mathcal{Z}_j \sim \mathcal{N}(0, \sigma_j^2)$, namely: 
    $
        P(\mathcal{Z}_j|{Pa}_j) \approx \mathcal{N}(0, \sigma_j^2).
    $
    Let $\mathcal{L}_{\text{KLD}}(\mathbf{X}, \tilde{\mathbf{X}})$ act as a penalty on deviations of $P(\mathcal{Z}_j|{Pa}_j)$ from $\mathcal{N}(0, \sigma_j^2)$. The gradient of $\mathcal{L}_{\text{KLD}}(\mathbf{X}, \tilde{\mathbf{X}})$ with respect to $G_A$ is:
    \[
        \nabla_{G_A} \mathcal{L}_{\text{KLD}}(\mathbf{X}, \tilde{\mathbf{X}}) = \sum_{j=1}^d \mathbb{E}_{{Pa}_j} \left[ \nabla_{G_A} \log \frac{P(\mathcal{Z}_j \mid {Pa}_j)}{\mathcal{N}(\mathcal{Z}_j; 0, \sigma_j^2)} \right].
    \] The term $\log \mathcal{N}(\mathcal{Z}_j; 0, \sigma_j^2)$ is quadratic in $\mathcal{Z}_j$, making $\nabla_{G_A} \mathcal{L}_{\text{KLD}}(\mathbf{X}, \tilde{\mathbf{X}})$ smooth and less sensitive to small variations in $G_A$. This prevents overfitting to noise in $X_j$, stabilizing the optimization of $f_j$. Hence, the KLD term can improve the overall stability of our model by approximating the implicitly generated distribution $P_{G_A}(\tilde{\mathbf{X}})$ to a normal (Gaussian) distribution. \\

    \noindent The KLD term also complements other loss terms. The adversarial loss $\mathcal{L}_{\text{adv}}(\mathbf{X}, \tilde{\mathbf{X}})$ ensures global alignment of $P(\mathbf{X})$ and $P_{G_A}(\tilde{\mathbf{X}})$, but does not directly enforce the additive Gaussian assumption. The MSE loss $\mathcal{L}_{\text{MSE}}(\mathbf{X}, \tilde{\mathbf{X}})$ focuses on point-wise alignment of $\mathbf{X}_i$ and $\tilde{\mathbf{X}}_i$, but does not account for statistical properties of $\mathcal{Z}_j$. The KLD regularization $\mathcal{L}_{\text{KLD}}(\mathbf{X}, \tilde{\mathbf{X}})$ explicitly enforces the Gaussianity of $\mathcal{Z}_j$, ensuring $\mathcal{Z}_j$ matches the additive Gaussian assumption, preventing $f_j$ from overfitting to non-Gaussian noise, thus concluding the proof. 
\end{proof}

\subsection{Proof of Proposition~\ref{prop4}}
\label{a4}
\primepropfour*

\begin{proof}
    
    The MMD loss term is
    \[
        \mathcal{L}_{\text{MMD}}(\mathbf{X}, \tilde{\mathbf{X}}) = \frac{1}{n}\sum\limits_{i \neq j}^{n}k(\mathbf{X}_i,\mathbf{X}_j) - \frac{2}{n}\sum\limits_{i \neq j}^{n}k(\mathbf{X}_i, \tilde{\mathbf{X}}_j) + \frac{1}{n}\sum\limits_{i \neq j}^{n}k(\tilde{\mathbf{X}}_i, \tilde{\mathbf{X}}_j).
    \] 
The gradient of $\mathcal{L}_{\text{MMD}}(\mathbf{X}, \tilde{\mathbf{X}})$ with respect to the parameters $\theta$ defining the model $G_A$ can be written as:
    \begin{equation*}
        \begin{aligned}
            \nabla_{\theta} \mathcal{L}_{\text{MMD}}(\mathbf{X}, \tilde{\mathbf{X}}) &= 2 ( \mathbb{E}_{\tilde{\mathbf{X}} \sim P_{G_A}(\tilde{\mathbf{X}})}[\nabla_{\theta}k(\tilde{\mathbf{X}}_i, \tilde{\mathbf{X}}_j)] \\ &- \mathbb{E}_{\mathbf{X} \sim P(\mathbf{X}), \tilde{\mathbf{X}} \sim P_{G_A}(\tilde{\mathbf{X}})}[\nabla_{\theta}k(\mathbf{X}_i, \tilde{\mathbf{X}}_j)] ),
        \end{aligned}
    \end{equation*} where $\tilde{\mathbf{X}} \sim P_{G_A}(\tilde{\mathbf{X}})$ are samples from the model-generated distribution, $\mathbf{X} \sim P(\mathbf{X})$ are samples from the true distribution and $k(\mathbf{X}, \tilde{\mathbf{X}})$ is a positive-definite kernel, often chosen as a Gaussian kernel or other characteristic kernel. \\
        
    \noindent The kernel function $k(\mathbf{X}, \tilde{\mathbf{X}})$ implicitly captures higher-order statistics of the distributions $P(\mathbf{X})$ and $P_{G_A}(\tilde{\mathbf{X}})$, including the internal consistency of the model distribution via the third term in $\mathcal{L}_{\text{MMD}}(\mathbf{X}, \tilde{\mathbf{X}})$, $\mathbb{E}_{\tilde{\mathbf{X}} \sim P_{G_A}(\tilde{\mathbf{X}})}[k(\tilde{\mathbf{X}}_i, \tilde{\mathbf{X}}_j)]$, which aligns model-generated samples $\tilde{\mathbf{X}}_i$ and $\tilde{\mathbf{X}}_j$ to ensure that the higher-order moments within $P_{G_A}(\tilde{\mathbf{X}})$ are coherent. It also allows alignment with the true distribution via the second term, $\mathbb{E}_{\mathbf{X} \sim P(\mathbf{X}), \tilde{\mathbf{X}} \sim P_{G_A}(\tilde{\mathbf{X}})}[k(\mathbf{X}_i, \tilde{\mathbf{X}}_j)]$.\\ 
    

    \noindent 
    $\mathcal{L}_{\text{MMD}}(\mathbf{X}, \tilde{\mathbf{X}})$ explicitly captures higher-order discrepancies through the kernel-induced feature mappings $k(.)$. This provides a complementary mechanism to adversarial losses, ensuring both global and fine-grained alignment between $P(\mathbf{X})$ and $P_{G_A}(\tilde{\mathbf{X}})$.
    \noindent Together, $\mathcal{L}_{\text{MMD}}(\mathbf{X}, \tilde{\mathbf{X}})$ and $\mathcal{L}_{\text{adv}}(\mathbf{X}, \tilde{\mathbf{X}})$ form a robust framework for distributional alignment, addressing both large-scale and higher-order mismatches, thus completing the proof.
    
\end{proof}

\subsection{Proof of Proposition~\ref{prop5}}
\label{a5}
\primepropfive*

\begin{proof}
    We split the proposition into two lemmas for identifiability under: 1) LiNGAM and ANM; 2) PNL, respectively.  \\

    \begin{lemma}
    \label{lem1}
        Under the additive noise model (ANM) or the linear non-Gaussian acyclic model (LiNGAM) assumption, the true DAG $\mathcal{G}_\mathcal{A}$ is uniquely identifiable from $P(\mathbf{X})$
        \[
            P(\mathbf{X}) \neq P^{\prime}(\mathbf{X}) \implies \mathcal{G}_\mathcal{A} \neq \mathcal{G^{\prime}}_\mathcal{A^{\prime}}. 
        \]
    \end{lemma}

    \begin{proof}
        Let the dataset $\chi$ consist of $X = \{X_1, ..., X_d\}$ data attributes, where each $X_j$ is generated under the ANM or LiNGAM assumption, both described using the following equation:
        \[
            X_j = f_j({Pa}_j) + \mathcal{Z}_j,
        \]where $f_j: \mathbb{R}^d \rightarrow \mathbb{R}$ are deterministic functions (nonlinear in ANM, linear in LiNGAM), $\mathcal{Z}_j \sim P(\mathcal{Z})$ are independent noise variables (non-Gaussian in LiNGAM, Gaussian in ANM), ${Pa}_j$ represents the set of direct parents of $X_j$ in the DAG.\\

        \noindent For both ANM and LiNGAM, the independence of $\mathcal{Z}_j$ from ${Pa}_j$ plays a crucial role: $\mathcal{Z}_j \independent {Pa}_j.$
        The independence of $\mathcal{Z}_j$ in the true DAG $\mathcal{G}_\mathcal{A}$ imposes strong constraints on the functional relationships in $\mathcal{G}_\mathcal{A}$:
        \[
            P(Z_j) = P_{\mathcal{Z}_j}(X_j - f_j({Pa}_j)),
        \] 
        where $\mathcal{Z}_j$ is the independent noise term. \\
        
        \noindent In the case when $\mathcal{G^{\prime}}_\mathcal{A^{\prime}} \neq \mathcal{G}_\mathcal{A}$, the functional relationships $f^{\prime}_j \in \mathcal{G^{\prime}}_\mathcal{A^{\prime}}$ must satisfy:
        \[
            P(Z^{\prime}_j) = P_{\mathcal{Z}^{\prime}_j}(X_j - f^{\prime}_j({Pa}^{\prime}_j)),
        \]
        where $\mathcal{Z}^{\prime}_j$ are the noise terms under $\mathcal{G^{\prime}}_\mathcal{A^{\prime}}$. \\

        \noindent However, when $\mathcal{G^{\prime}}_\mathcal{A^{\prime}} \neq \mathcal{G}_\mathcal{A}$, the new functional relationships $f^{\prime}_j$ will be different from $f_j$ in the true DAG. Furthermore, the new noise terms $\mathcal{Z}^{\prime}_j$ will not remain independent of ${Pa}^{\prime}_j$ because the independence of $\mathcal{Z}_j$ is specific to the true causal structure in $\mathcal{G}_\mathcal{A}$. This implies that $\mathcal{G^{\prime}}_\mathcal{A^{\prime}}$ cannot satisfy the independence assumptions simultaneously with $\mathcal{G}_\mathcal{A}$, leading to a contradiction. \\


        \noindent Hence, under the assumptions of the ANM with nonlinear functions and independent noise or the LiNGAM model with linear functions and non-Gaussian noise, there exists no other DAG $\mathcal{G^{\prime}}_\mathcal{A^{\prime}} \neq \mathcal{G}_\mathcal{A}$ that can generate the same observational data distribution $P(\mathbf{X})$. Therefore, the true DAG $\mathcal{G}_\mathcal{A}$ is uniquely identifiable only from $P(\mathbf{X})$, thus concluding the proof.\\ 
    \end{proof}


    \begin{lemma}
    \label{lem2}
        Under the Post-Nonlinear (PNL) model assumption, there exists an identifiable DAG $\mathcal{G}_\mathcal{A}$ that generates the observed joint distribution of the data variables $\{X_1,...,X_d\}$.
    \end{lemma}

    \begin{proof}
        Let $\chi$ be a dataset consisting of $\{X_1,...,X_d\}$ data attributes, where each $X_j$ is described as follows:
        \[
            X_j := g_j(f_j({Pa}_j) + \mathcal{Z}_j), \forall j, \mathcal{Z}_j \independent f_j({Pa}_j), \mathcal{Z}_j \sim \mathcal{N}(\mu, \sigma^2_j),
        \]where ${Pa}_j$ is the set of parent nodes for $X_j$, $f_j$ are nonlinear functions modeling parent contributions, $g_j$ is a nonlinear function applied post-summation and $\mathcal{Z}_j$ is an independent Gaussian noise term, satisfying $\mathcal{Z}_j \independent {Pa}_j$. \\

        \noindent Moreover, let $N_j$ be the input to $g_j$ such that:
        \[
            N_j = f_j({Pa}_j) + \mathcal{Z}_j.
        \] Under the assumption that ${Pa}_j$ is the true parent set, the noise term $\mathcal{Z}_j$ is independent of its parents:
        \[
            \mathcal{Z}_j \independent {Pa}_j.
        \] In addition, $g_j$ does not affect the independence structure. Thus, for the true set of parents ${Pa}_j$, the residual noise $\mathcal{Z}_j$ remains independent of the parent variables. \\

        \noindent Under this setting, the statistical relationship between $X_j$, its parents, and the residual noise satisfies specific invariances:
        \[
            P(X_j, {Pa}_j) = P(X_j|{Pa}_j)P({Pa}_j),
        \] where $P(X_j|{Pa}_j)$ is derived from the PNL structure. \\

        \noindent Now, consider any alternative parent set ${Pa}^{\prime}_j \neq {Pa}_j$. For this incorrect set of parents, the residual noise $\mathcal{Z}_j$ is reconstructed as:
        \[
            \mathcal{Z}_j = N_j - f_j({Pa}^{\prime}_j).
        \]In this case, the core independence condition $\mathcal{Z}_j \independent {Pa}^{\prime}_j$ is violated. Therefore, when the parent set is incorrect, the residual noise $\mathcal{Z}_j$ will exhibit statistical dependencies with the variables in ${Pa}^{\prime}_j$. This implies that the conditional distribution $P(X_j|{Pa}^{\prime}_j)$ cannot reproduce the same invariance due to the introduced dependencies, thus concluding the proof.
    \end{proof}

    \begin{corollary}
    \label{col1}
        Under Lemmas \ref{lem1} and \ref{lem2}, the uniqueness property of $G_A$ allows us to reconstruct the generative process of $\mathbf{X}$.
    \end{corollary}

    \noindent Corollary \ref{col1} implies that under the causal model assumption employed in DAGAF, we can accurately generate synthetic samples with preserved causal structures, which is only possible if $G_A = \mathcal{G}_\mathcal{A}$. In turn, this implies that the implicitly generated distribution $P_{G_A}(\tilde{\mathbf{X}})$ is the same as the observed distribution $P(\mathbf{X})$. Therefore, we have demonstrated that there exists a single unique DAG capable of constructing the input data distribution, thus concluding the proof. 
\end{proof}

\section{Ablation study}
\label{ab}

We conducted an ablation study to determine the optimal configuration of the terms in the loss function for Step 1. We carried out nine experiments on the Sachs, ECOLI70, MAGIC-IRRI and ARTH150 datasets  under the ANM assumption, testing various combinations of loss terms. These continuous (Gaussian) datasets are available at \url{https://www.bnlearn.com/bnrepository/}. All cases include the Wasserstein-1 distance. The first configuration is labeled "\textbf{w/o recon loss}", where the reconstruction loss with its regularization is excluded from the training algorithm. The rest are named according to the terms included in the reconstruction loss, such as MSE \cite{Bickel2015MathematicalSB} and NLL \cite{GrofOnTM}. We also tested combinations of additional terms such as MMD \cite{Tolstikhin2016MinimaxEO} and KLD \cite{Kullback1951OnIA}. The results of this study are shown in Table \ref{tab:table8}.

\begin{table}[!ht]
\centering
\caption{DAGAF ablation study}
\label{tab:table8}
\begin{tabular}{@{}ccccc@{}}
\toprule
\multirow{2}{*}{Loss function} & \multicolumn{4}{c}{SHD}                               \\ \cmidrule(l){2-5} 
                               & Sachs      & ECOLI70     & MAGIC-IRRI  & ARTH150      \\ \midrule
w/o recon loss                 & 21         & 115         & 163         & 377         \\
recon loss (MSE)               & 14         & 91          & 117         & 288          \\
recon loss (NLL)               & 16         & 106         & 132         & 320          \\
MSE + MMD                      & 10         & 57          & 80          & 189          \\
NLL + MMD                      & 14         & 91          & 117         & 288          \\
MSE + KLD                      & 12         & 69          & 99          & 221          \\
NLL + KLD                      & 12         & 69          & 99          & 221          \\
MSE + KLD + MMD                & \textbf{9} & \textbf{52} & \textbf{71} & \textbf{175} \\
NLL + KLD + MMD                & 11         & 60          & 86          & 197          \\ \bottomrule
\end{tabular}%
\end{table}

The ablation study reveals the optimal combination of loss terms for our method. As shown in Table \ref{tab:table8}, the best set of loss terms in Step 1 includes MSE, KLD, MMD, and adversarial training. Further details on each of these metrics and regularization are provided in Section \ref{sec:3.1}. 

\section{Sensitivity analysis}
\label{ac}

To ensure model robustness, we perform a sensitivity analysis to examine how the training responds to different hyper-parameter settings. This study measures the accuracy of DAG reconstruction (i.e., SHD) under various hyper-parameters, including learning and dropout rates (\textbf{lr, dropout}), noise vector and batch sizes (\textbf{z-size, batch-size}). We begin with a baseline setting of \textbf{lr = 0.001, dropout = 0.5, z-size = 1, batch-size = 100}, then modify each value individually to observe the changes in SHD. All experiments were conducted on the Sachs dataset \color{black} by applying the ANM causal model, and the results are presented in Table \ref{tab:table9}.

\begin{table}[!ht]
\centering
\caption{DAGAF sensitivity analysis }
\label{tab:table9}
\begin{tabular}{@{}cc@{}}
\toprule
\multirow{2}{*}{Hyper-parameters}                       & Sachs Dataset \\ \cmidrule(l){2-2} 
                                                        & SHD           \\ \midrule
lr = 3e-3, dropout = 0.5, z-size = 1, batch-size = 100  & 9             \\
lr = 3e-3, dropout = 0.0, z-size = 1, batch-size = 100  & 10            \\
lr = 3e-3, dropout = 0.5, z-size = 2, batch-size = 100  & 10            \\
lr = 3e-3, dropout = 0.5, z-size = 5, batch-size = 100  & 11            \\
lr = 3e-3, dropout = 0.5, z-size = 1, batch-size = 500  & 9             \\
lr = 3e-3, dropout = 0.5, z-size = 1, batch-size = 1000 & 10            \\
lr = 2e-4, dropout = 0.5, z-size = 1, batch-size = 100  & 11            \\
lr = 1e-3, dropout = 0.5, z-size = 1, batch-size = 100  & 12            \\ \bottomrule
\end{tabular}%
\end{table}

The results from Table \ref{tab:table9} indicate that lowering the learning and dropout rates significantly affects the performance of our model. On the other hand, increasing the size of the noise vector and the input data batch results in only minor variations in the accuracy of the algorithm.


\section{Additional results}
\label{ad}
In this section, we present further examples to reinforce the data quality analysis discussed in Section \ref{ref:sec4.4}. We provide real-synthetic statistical comparisons for all features (Table \ref{tab:pvalue}), additional visualizations of the synthetic feature distributions (Figure \ref{fig:ad-stat}), and the remaining machine learning regression results (Figure \ref{fig:ad-featimp}).

\begin{figure}[!htp]
\centering	 
    \begin{subfigure}[t]{0.49\textwidth}
        \includegraphics[width=\textwidth]{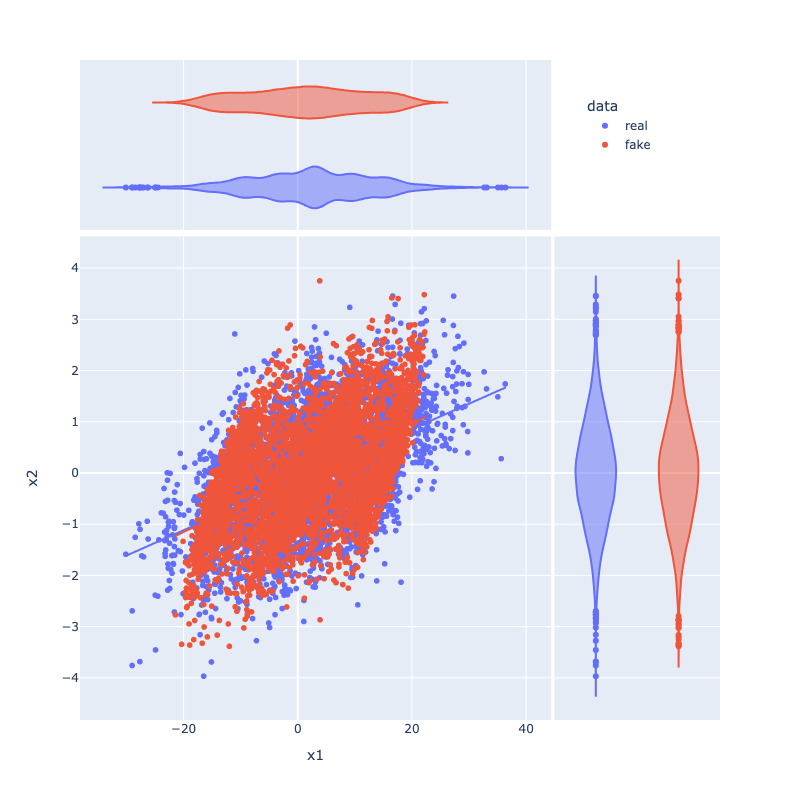} 
    \end{subfigure}
    \begin{subfigure}[t]{0.49\textwidth}
        \includegraphics[width=\textwidth]{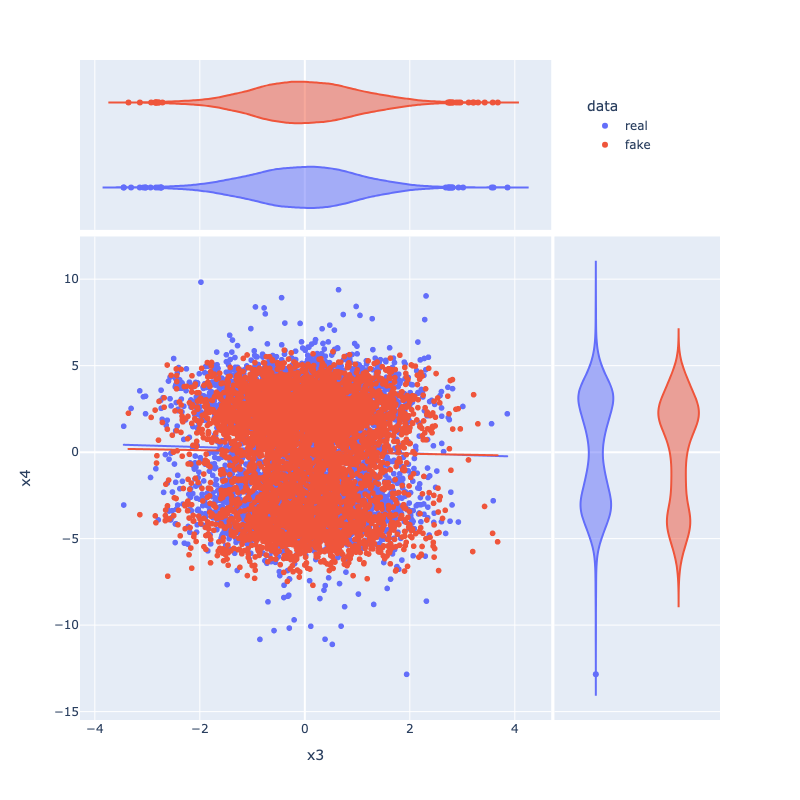}
    \end{subfigure}
    \begin{subfigure}[t]{0.49\textwidth}
        \includegraphics[width=\textwidth]{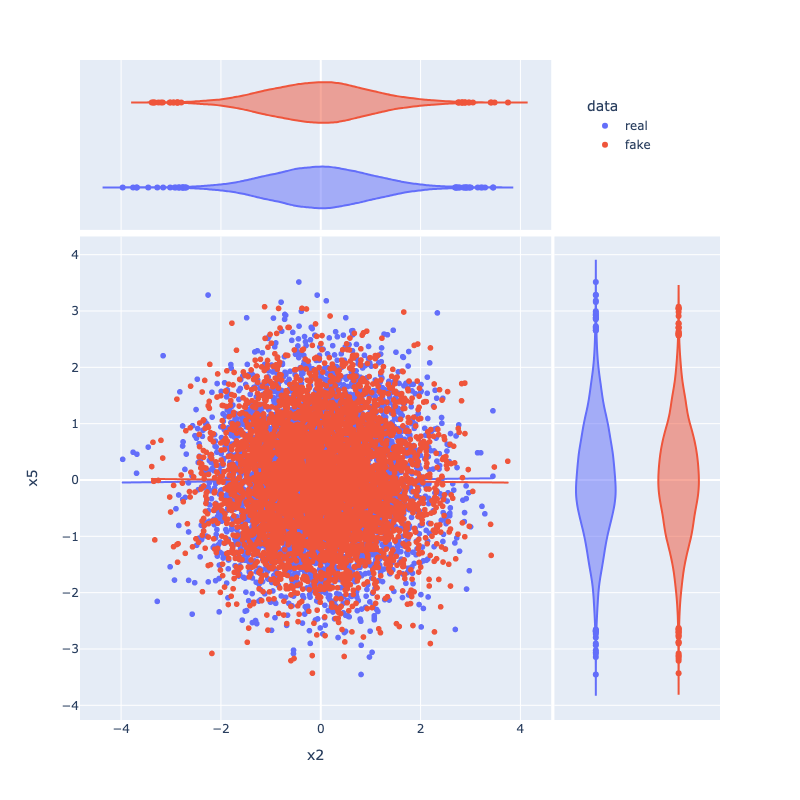}
    \end{subfigure}
     \begin{subfigure}[t]{0.49\textwidth}
        \includegraphics[width=\textwidth]{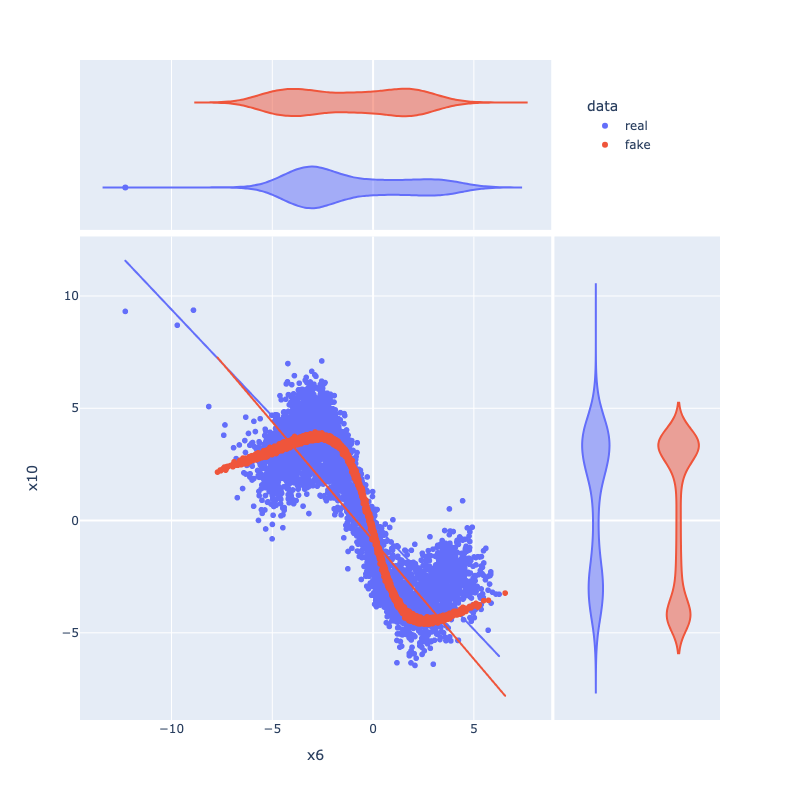}
    \end{subfigure}
\caption{Further examples of the synthetic joint and marginal distributions for our method on the dataset presented in Section \ref{ref:sec4.4}. We observe multiple cases with different distribution shapes. Additionally, we depict one case of severe mode collapse (bottom) in the produced data from DAGAF. 
}
\label{fig:ad-stat}
\vspace{-5mm}
\end{figure}

\begin{figure}[!htp]
\vspace{-5mm}
\centering	 
    \includegraphics[width=\textwidth]{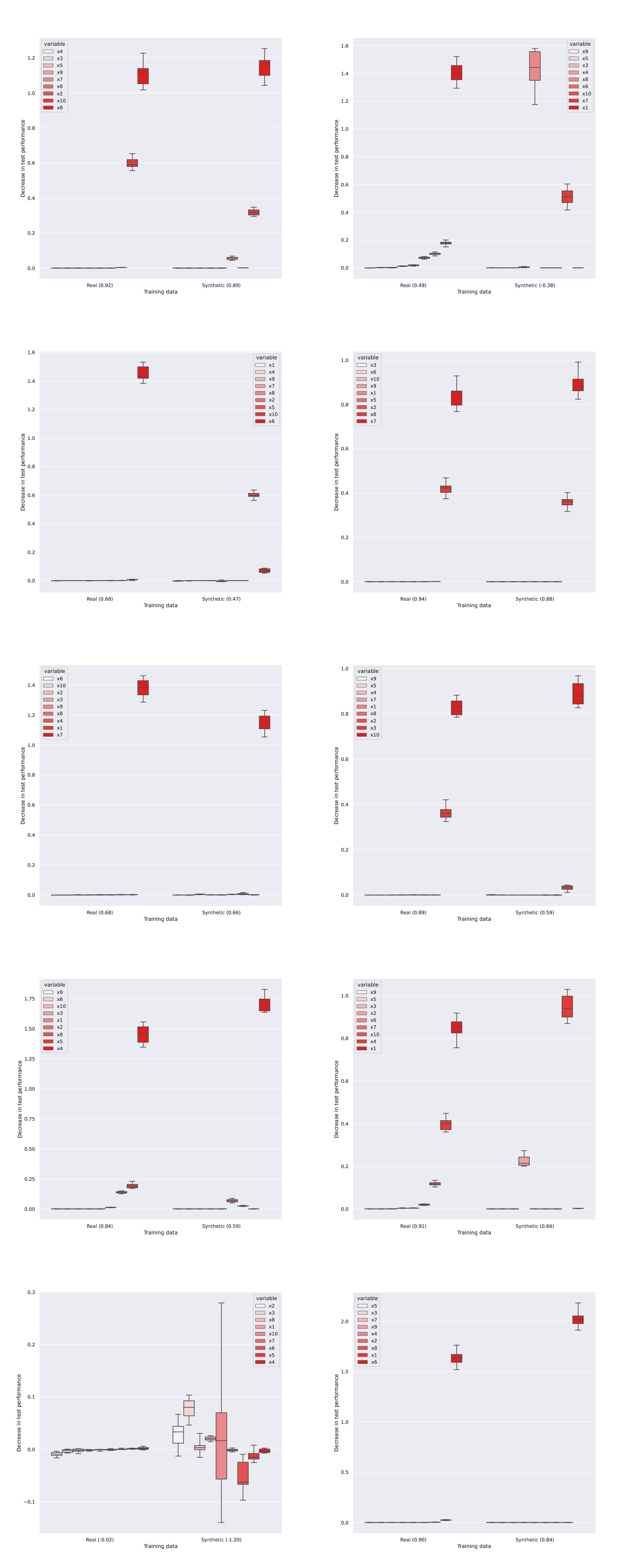} 
    %
\caption{Remaining examples of feature importances to supplement the results in Section \ref{ref:sec4.4}. We observe some failure cases, where the synthetic features differ significantly from their real counterparts. 
}
\label{fig:ad-featimp}
\vspace{-5mm}
\end{figure}

\begin{table}[!htp]
\centering
\caption{Mann-Whitney t-test results for all real and synthetic features to supplement Figure \ref{fig:2}. We observe some failure cases, where the real and synthetic features differ significantly ($p<0.05$).}
\label{tab:pvalue}
\begin{tabular*}{0.5\textwidth}{@{\extracolsep{\fill}}>{\centering\arraybackslash}p{0.25\textwidth} >{\centering\arraybackslash}p{0.25\textwidth}@{}}
\toprule
Feature & $p$-value \\
\midrule
x1  & 7.7952e-07            \\
x2  & 0.5004            \\
x3   & 0.1683            \\
x4  &  0.0020   \\
x5 & 0.8563            \\
x6  & 0.9127            \\
x7 & 0.0364            \\
x8  & 0.1747            \\
x9  & 0.2089            \\
x10  & 6.4502e-26            \\ 
\bottomrule
\end{tabular*}%
\end{table}

\section{DAGAF pseudo-code}
\label{ae}
\begin{algorithmic}[1]
\scriptsize
\State $\lambda \gets 0$, $c \gets 1$
\State $current\_h(A^{L_0}(f)) \gets \infty$, $h\_tol \gets 1e-8$ 
\State $k\_max\_iter \gets 100$, $epochs \gets 300$
\For{k $< k\_max\_iter$}
\While{$c < 1e+20$} 
\For{epoch $< epochs$}
\State \If{$pnl == True$} \Comment{The beginning of the Causal Discovery (CD) Step}
        \State $\tilde{X} := \{g_1(f_1({Pa}_1;W^1_1,...,W^L_1) + \mathcal{Z}_1),...,g_d(f_d({Pa}_d;W^1_d,...,W^L_d) + \mathcal{Z}_d)\}$ 
    \Else
        \State $\tilde{X} := \{f_1({Pa}_1;W^1_1,...,W^L_1) + \mathcal{Z}_1,...,f_d({Pa}_d;W^1_d,...,W^L_d) + \mathcal{Z}_d\}$ 
    \EndIf \\

    \State DiscLoss = $L_D(\mathbf{X}, \tilde{\mathbf{X}})$
    \State GenLoss = $L_G(\mathbf{X})$
    \State RecLoss = $L_{REC}(\mathbf{X}, \tilde{\mathbf{X}}) + \frac{c}{2}|h(A^{L_0})|^2 + \lambda h(A^{L_0})$
    \State PnlLoss = $L_{PNL}(X^{-1}, \tilde{\mathbf{X}})$ \Comment{if PNL is assumed} \\

    \State DiscGradients = DiscLoss.backward()
    \State GenGradients = GenLoss.backward()
    \State RecGradients = RecLoss.backward()
    \State PnlGradients = PnlLoss.backward() \Comment{if PNL is assumed} \\
    
    \State DiscParameters = DiscParameters - $1e-3$ * DiscGradients 
    \State GenParameters = GenParameters - $1e-3$ * GenGradients
    \State RecParameters = RecParameters - $1e-3$ * RecGradients
    \State PnlParameters = PnlParameters - $1e-3$ * PnlGradients \Comment{if PNL is assumed} \\

    \State $DS\{W^{L_0}\} \gets CD\{W^{L_0}\}$ \Comment{Parameter transfer between steps}
    
\State \If{$pnl == True$} \Comment{The beginning of the Data Synthesis (DS) Step}
        \State $\tilde{X} := \{g_1(G_1({Pa}_1;W^1_1,...,W^L_1) + Z_1),...,g_d(G_d({Pa}_d;W^1_d,...,W^L_d) + Z_d)\}$ 
    \Else
        \State $\tilde{X} := \{G_1({Pa}_1;W^1_1,...,W^L_1) + Z_1,...,G_d({Pa}_d;W^1_d,...,W^L_d) + Z_d\}$ 
    \EndIf \\

    \State DiscLoss = $L_D(\mathbf{X}, \tilde{\mathbf{X}})$
    \State GenLoss = $L_G(Z)$ \\

    \State DiscGradients = DiscLoss.backward()
    \State GenGradients = GenLoss.backward() \\

    \State DiscParameters = DiscParameters - $1e-3$ * DiscGradients 
    \State GenParameters = GenParameters - $1e-3$ * GenGradients \\
\EndFor
\If{$h(A^{L_0}(f)) > 0.25$}
\State $c \gets c * 10$
\Else
\State $break$
\EndIf 
\EndWhile 
\State $current\_h(A^{L_0}(f)) \gets h(A^{L_0}(f))$
\State $\lambda \gets c * current\_h(A^{L_0}(f))$
\If{$current\_h(A^{L_0}(f)) \leq h\_tol$}
\State $break$
\EndIf 
\EndFor
\end{algorithmic}

\end{appendices}


\section*{Declarations}

\textbf{Competing interests} The authors declare that they have no competing financial or non-financial interests in relation to this work. \\

\noindent \textbf{Ethical and informed consent for data used} Not applicable. \\

\noindent \textbf{Data availability} The authors confirm that all data (with their corresponding repository and citation links) relevant to the research carried out to support their work are included in this article. \\

\noindent \textbf{Authors contribution} Hristo Petkov (First Author) is responsible for software development, theoretical analysis, conducting causal experiments and draft preparation. Calum MacLellan (Second Author) is responsible for performing data synthesis experiments and draft revision. Feng Dong (Third Author) is responsible for overall draft proofreading and refactoring. \\

\noindent \textbf{Funding} The authors declare that their work has been funded by the United Kingdom Medical Research Council (Grant Reference: MR/X005925/1) throughout the duration of their associated research project (Virtual Clinical Trial Emulation with Generative AI Models, Duration: Sept 2022 – Feb 2023).

\bibliography{sn-bibliography.bib}

\end{document}